\documentclass[journal]{IEEEtran}
\renewcommand{\paragraph}[1]{\noindent\textbf{#1.}}

\usepackage{url}
\usepackage[utf8]{inputenc} 
\usepackage[T1]{fontenc}    
\usepackage{booktabs}       
\usepackage[ruled]{algorithm2e}
\usepackage{amsmath,amsfonts,bm}
\usepackage{graphicx}
\usepackage{microtype}      
\usepackage{multirow}
\usepackage{nicefrac}       

\RestyleAlgo{ruled}
\usepackage{subfigure}
\usepackage{wrapfig}
\usepackage{cite}

\usepackage{comment}
\usepackage{times}  
\usepackage{helvet}  
\usepackage{courier}  

\usepackage[hidelinks,colorlinks]{hyperref}
\usepackage{xcolor}         
\usepackage[capitalize]{cleveref}
\usepackage{pppkg}

\def\Eqref#1{Equation (\ref{#1})}

\newcommand{\RNum}[1]{\uppercase\expandafter{\romannumeral #1\relax}}

\ifCLASSINFOpdf
\else
\fi

\begin{document}
\title{PMGDA: A Preference-based Multiple Gradient Descent Algorithm}

\author{
    Xiaoyuan Zhang ~\IEEEmembership{Member,~IEEE},
    Xi Lin,
    Qingfu Zhang ~\IEEEmembership{Fellow,~IEEE}
    \thanks{All authors are with the Department of Computer Science, City University of
        Hong Kong (email: xzhang2523-c@my.cityu.edu.hk, qingfu.zhang@cityu.edu.hk). Corresponding to Prof. Qingfu Zhang. 
}}

\maketitle
\begin{abstract}
It is desirable in many multi-objective machine learning applications, such as multi-task learning with conflicting objectives and multi-objective reinforcement learning, to find a Pareto solution that can match a given preference of a decision maker. These problems are often large-scale with available gradient information but cannot be handled very well by the existing algorithms. To tackle this critical issue, this paper proposes a novel predict-and-correct framework for locating a Pareto solution that fits the preference of a decision maker. In the proposed framework, a constraint function is introduced in the search progress to align the solution with a user-specific preference, which can be optimized simultaneously with multiple objective functions. Experimental results show that our proposed method can efficiently find a particular Pareto solution under the demand of a decision maker for standard multiobjective benchmark, multi-task learning, and multi-objective reinforcement learning problems with more than thousands of decision variables.

Code is available at: \url{https://github.com/xzhang2523/pmgda} \footnote{Our code is current provided in the \texttt{pgmda.rar} attached file and will be open-sourced after publication.}.

\end{abstract}
	
 
\IEEEpeerreviewmaketitle

\section{Introduction}
\IEEEPARstart{R}{ecently}, Multi-Objective Optimization (MOO) concepts and methodologies have received growing attention in the machine learning community since many machine learning applications can be naturally modeled as MOO problems ~\cite{jin2006multi,sener2018multi,zhou2021multiple}.
For example, a face recognition system has to balance prediction accuracy on different attributes of the face~\cite{ruchte2021cosmos,liu2018large}. A recommendation system needs to balance several metrics, such as accuracy, novelty, and diversity of products \cite{zheng2021multi,haolun2022multi}. Very often, no single best solution can optimize all these objectives simultaneously. Instead of a single optimal solution, there exists a set of Pareto-optimal solutions with different optimal trade-off ratios.


Modern machine learning systems usually use deep neural networks as the core model. MOO problems in these machine learning systems~\cite{sener2018multi,lin2019pareto,lin2020controllable,liu2018large,ruchte2021cosmos} often share the following characteristics:
\begin{itemize}

\item Large scale decision space: Since a deep network model can often have thousands to millions of parameters to tune~\cite{lecun2015deep,goodfellow2016deep},  it is inevitable that these optimization problems are of large scale. Evolutionary algorithms cannot handle them very efficiently. 

\item Almost everywhere smooth: Most widely used loss functions in deep learning are smooth, and their gradients are easy to obtain by backpropagation. In addition, these loss functions can usually be optimized to their global optimal solutions by gradient methods because neural network models are often overparameterized. \cite{lopez-paz2018easing, allen2019convergence}. 

\end{itemize}

\begin{table*}
\centering
\caption{Comparison of PMGDA with other gradient-based MOO methods.} \label{tab:compare}
\begin{tabular}{lll}
\toprule
Method & Solution type & Shortcomings \\
\midrule
EPO \cite{pmlr-v119-mahapatra20a} & Exact Pareto solutions.  & \begin{tabular}[c]{@{}l@{}}(1) EPO is only designed for minimization tasks. \\ (2) EPO can only deal with positive objectives. \\ (3) Solution precision of EPO is bad on a number of MOO problems. \end{tabular} \\
\midrule
COSMOS \cite{ruchte2021cosmos} & Approximate exact Pareto solution. & Solution quality is heavily depended on its control parameter $\mu$ and this parameter is difficult to tune.  \\
\midrule
PMTL \cite{lin2019pareto} & Pareto solutions restricted in sector Areas. & Challenging to apply for objectives $m$>2 due to sector divisions in 3-D (or more) space is difficult. \\
\midrule
mTche \cite{ma2017tchebycheff} & Exact Pareto solutions. & 
\begin{tabular}[c]{@{}l@{}}(1) mTche is inefficient for MO problems with more than two objectives ($m>2$). \\ (2) Final performance is bad on complex MO reinforcement learning problems.  \end{tabular} \\
\midrule
PMGDA & \begin{tabular}[c]{@{}l@{}}Pareto solutions under inequality/\\equality  constraints. \\ PMGDA is particular designed for \\ exact and ROI Pareto solutions.  \end{tabular} & PMGDA solves a LP/QP (Similar to EPO/PMTL), taking more time than mTche.   \\
\bottomrule
\end{tabular}
\end{table*}

Due to these characteristics, gradient-based methods are particularly well-suited for solving MOO problems in machine learning. It is, however, very costly to find a set of distinct Pareto solutions (e.g., deep neural networks) to approximate the entire Pareto front accurately. Some efforts have been made to design multi-objective gradient algorithms to find one or a few Pareto solutions~\cite{sener2018multi,lin2019pareto,liu2021conflict,zhou2021multiple,milojkovic2019multi} with successful applications on multi-task learning, image classification, face recognition, and other scene understanding tasks~\cite{kendall2018multi,sener2018multi} based on the Multiple Gradient Descent Algorithm (MGDA)~\cite{fliege2000steepest}. However, a significant limitation of most of these gradient algorithms \cite{sener2018multi, lin2019pareto, fliege2000steepest, desideri2012mutiple} is their lack of consideration for the decision maker's preferences. 

With a very limited number of solutions (only one or a few), it is of great importance to develop algorithms for finding Pareto solutions satisfying user's specialized demands. Some initial studies have explored approaches for finding Pareto objective solutions that align with user preferences. For instance, Pareto Multi-Task Learning (PMTL)~\cite{lin2019pareto} identifies Pareto solutions within specific sector regions, and the COSMOS method~\cite{ruchte2021cosmos} incorporates a penalty term to approximate a Pareto solution near a preference vector. The Exact Pareto Optimization (EPO) method from Mahapatra et al.~\cite{pmlr-v119-mahapatra20a} focuses on locating solutions along a preference vector, claiming to find the `exact' solution with preference. However, these methods are still nascent and have their own disadvantages. For example, PMTL~\cite{lin2019pareto} struggles with efficiency for more than three objectives, and COSMOS~\cite{ruchte2021cosmos} is limited to approximate solutions. EPO employs a complex and specific search mechanism to divide objectives into different sets with different treatments. This approach sometimes fails to find the preferred solutions~\footnote{If the linear programming problem defined in EPO fails, EPO relies on optimizing a linear combination of objectives.} and performs poorly with more than two objectives. A comprehensive comparison with previous methods is presented in \Cref{tab:compare}.

Different from previous works, this paper introduces a general Preference-based Multiple Gradient Descent Algorithm, termed PMGDA, to effectively address user-specific preferences within the MGDA framework~\cite{fliege2000steepest}. The preference assignment in PMGDA is flexible and general, which includes but is not limited to finding the exact Pareto solution~\cite{pmlr-v119-mahapatra20a}. The proposed PMGDA has a predict-and-correct approach to optimize the preference and all objective functions simultaneously. This paper makes the following main contributions:

\begin{enumerate}

    \item We propose a novel framework for identifying Pareto solutions that align with a wide range of general preference functions, where finding exact Pareto solutions is a special case. Our approach is the first to effectively find Pareto solutions that meet specific region of interest (ROI) \cite{li2018integration,thiele2009preference} requirements.
    
    \item We develop efficient algorithms for both the prediction and correction steps for PMGDA, which makes it suitable to tackle large-scale multi-objective deep neural network training problems.
    
    \item We rigorously test PMGDA on well-known synthetic challenges, deep multi-task learning, and multiobjective reinforcement learning problems. PMGDA surpasses previous approaches in terms of solution precision and computational efficiency for finding exact Pareto solutions, while it is the first approach to successfully find Pareto solutions that satisfy the ROI constraint.  

\end{enumerate}

\section{Background} \label{sec_bg}
\ssec{Notations}
In this paper, the following notations are used (refer to \Cref{tab:notation}): Vectors are denoted using bold letters (e.g. $\vv$). Matrices are denoted by an uppercase bold letters (e.g., $\vG$).
The index set $\{1,\ldots,m\}$ is shorthanded by $[m]$.
We use three symbols to compare the relationship between two vectors $\vx$ and $\vy$:

For two vectors, $\vx$ and $\vy$, $\vx \prec \vy$ if $x_i < y_i$ for all $i \in [m]$.
$\vx \preceq \vy$ if $x_i \leq y_i$ for all $i \in [m]$.
$\vx \precs \vy$ if there exists $j \in [m]$ such that $x_j < y_j$ and $x_i \leq y_i$ for all $i \in [m] \setminus j$.
In this article, $x_i$ denotes the $i$-th entry of vector $\vx$, and $\vx^{(i)}$ denotes the $i$-th vector in a set of vectors.

\begin{table}[h!]
    \centering
    \caption{Notations used in this paper.} \label{tab:notation}%
    \begin{tabular}{ll}
        \toprule
        Variable & Definition \\
        \midrule
        $\vth$ & The decision variable. \\
        $n$     & The dimension of the decision variables. \\
        $m$     & The number of objectives. \\
        $[m]$     & The index set $\lbr{1, \ldots, m}$. \\
        $\Delta^{m-1}$ & The (m-1)-simplex, $\lbr{\vy|\sum_{i=1}^m y_i=1, y_i \geq 0 }$. \\
        $\vd$     & The search direction. \\
        $\vv$     & The direction (argument) used in the optimization problem. \\
        $\vlam$ & An $m$-D preference vector.  \\
        $\vL(\vth)$ & The objective functions. $L(\vth)=[L_1(\vth),\ldots,L_m(\vth)]$. \\
        $h(\vth)$ & The constraint function for the $\vlam$-exact solution.   \\
        $\hat{\mtx x}$ & $\hat{\mtx x}=\frac{\vx}{\norm{\vx}}$. I.e, $\hat{\mtx x}$ is a unit vector. \\
        $\mtx \mu, \nu$ & Lagrange multipliers. \\
        \bottomrule
    \end{tabular}%
\end{table}%

\ssec{Basic MOO Definitions} \label{sec:moobasic}
A MOO problem can be expressed as the following formulation: 
\begin{equation}
    \label{basic_mop_problem}
    \min_{\vth \in \Theta \subset \R^n} \vL (\vth)=(L_1(\vth), L_2(\vth),\ldots,L_m(\vth)) ,
\end{equation}
where $\vth$ is a $n$-D decision variable, and $\Theta$ is its decision space. Very often, no single solution can optimize all the objectives at the same time. Instead, there exists a set of solutions, called Pareto solutions, which can represent different optimal trade-offs among the objectives. The following MOO definitions are used in this paper. We have the following definition:

\begin{definition}[Dominance \cite{miettinen2012nonlinear}(Chap. 2.2)]
    Given two candidate solutions $\vth^{(a)}, \vth^{(b)} \in \Theta$, for a minimization problem, we say solution $\vth^{(a)}$ dominates $\vth^{(b)}$, if $\vL(\vth^{(a)}) \precs \vL(\vth^{(b)})$.   
\end{definition}

\begin{definition}[Pareto Optimality \cite{miettinen2012nonlinear}(Chap. 2.2)] 
    A solution is  Pareto optimal if no other solution belongs to $\Theta$ dominates it. The set of all the  Pareto solutions is called the Pareto set (\texttt{PS}); its image in the objective space is called the Pareto front ($\Tau$), $\Tau= f \circ \mathtt{PS}$.  
\end{definition}

\begin{definition}[Pareto Stationary Solution]
For objectives \( L_1, \ldots, L_m \) under the constraint \( h(\vth) \), a solution \( \vth \) is deemed Pareto stationary if it satisfies the following condition: there are no vectors \( \mu_1, \ldots, \mu_m \), and \( \nu \) (not all zero) for which
\begin{equation}
\sum_{i=1}^m \mu_i \nabla L_i + \nu \nabla h = 0, \quad \text{with} \quad \mu_i \geq 0. 
\end{equation}
\end{definition}

\begin{figure}%
    \centering
    \sfig[Exact Pareto solution.]{\ig[width= \fwdt \tw]{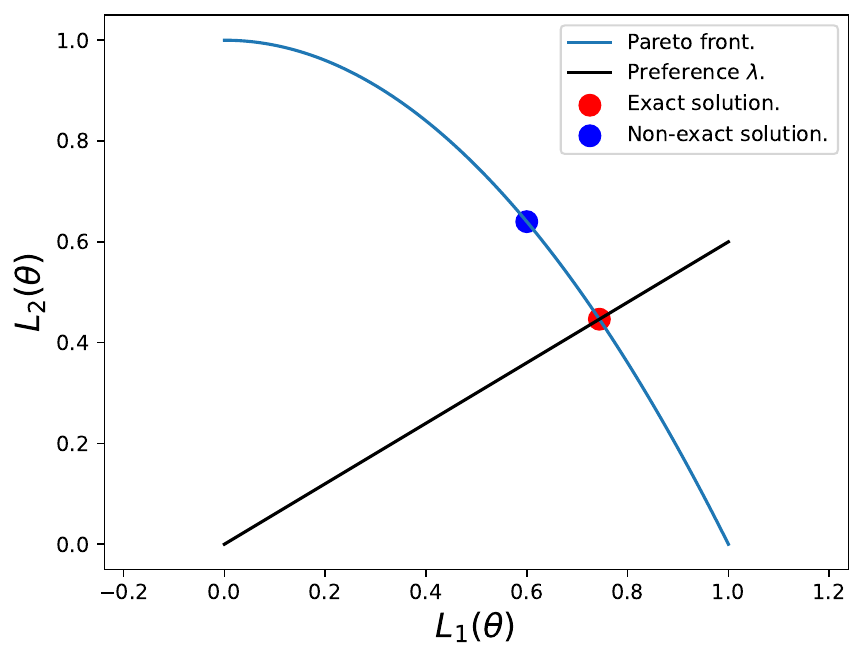}}
    \sfig[Pareto solutions with a region of interest (ROI).]{\ig[width= \fwdt \tw]{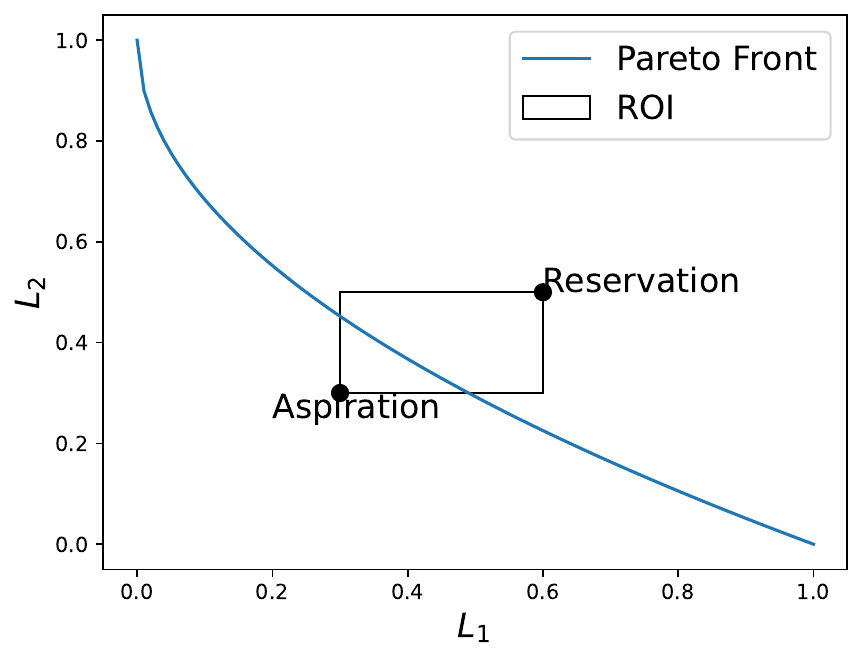}}
    \caption{Two particular cases of preference-based MOO. The left figure: the exact Pareto solution aligns with a given preference vector. The right figure: ROI Pareto solutions satisfy the ROI constraint. 
    } \label{fig:prefmoo}
\end{figure}
 
\noindent A solution is Pareto stationary means that, this solution can not be locally improved, which is a necessary condition for Pareto solution. Since a Pareto set could contain more than one single solution, and the possible infinite Pareto solutions are not comparable (e.g., non-dominated) with each other, therefore, it is necessary to make a further decision among the Pareto solutions \cite{roy2023optimization}. Preference-based MOO aims to search for a Pareto solution that satisfies a user-defined preference function $h(\vth)=0$. A direct application is to search the so-called `exact' Pareto solutions.  

\begin{definition}[`Exact' Pareto Solution]
     Let $\vlam=(\lambda_1, \ldots, \lambda_m)$ be a positive preference vector, i.e., $\sum_{i=1}^{m} \lambda_i=1$ and $\lambda_i>0$, $\forall i \in [m]$. A solution $\vth$ is called an `exact' Pareto solution \cite{pmlr-v119-mahapatra20a} for the preference $\vlam$ if this solution is Pareto optimal and satisfies:
    \begin{equation}
      \frac{L_1(\vth)}{\lambda_1}=\frac{L_2(\vth)}{\lambda_2}=\ldots=\frac{L_m(\vth)}{\lambda_m}.
      \label{eqn:lambda}
    \end{equation}
\end{definition}

Another application of preference-based MOO is that, in many cases, particularly when dealing with large problem sizes, the user may only be interested in a specific region of interest (ROI) on the Pareto front \cite{li2018integration}. This falls under the category of preference-based MOO \cite{thiele2009preference}. Previous literature has defined the ROI using the reservation point $\overline{\mtx y}$ and the aspiration point $\underline{\mtx y}$. The ROI can be represented as $\mathtt{ROI} = \{ \mtx y | \underline{\mtx y} \preceq \mtx y \preceq \overline{\mtx y} \}$ \cite{gonzalez2021preference}. Illustrations of `exact' solution and ROI are shown in Figure \ref{fig:prefmoo}.

\subsection{The Multiple Gradient Descent Algorithm (MGDA)}
In this section, we briefly review MGDA ~\cite{fliege2000steepest}, which serves as a fundamental tool of our method. MGDA extends the single objective steepest gradient descent algorithm for solving MOO problems. It updates the parameters as follows: $\vth_{k+1} = \vth_k + \eta \vd$, where $\eta$ is the learning rate and $\vd$ is the MGDA search direction. The key idea of MGDA is to find a valid gradient direction to decrease all objectives $L_i(\vth)$ at the same time. 
One way for computing $\vd$ is as follows \cite{desideri2012mutiple}: 
\bee \label{eqn:mgda1}
    (\vd, \alpha^*)= & \argmin_{(\vv, \alpha) \in \mathbb{R}^{m+1}} \alpha + \frac{1}{2}  \norm{ \vv }^2, \\
    \st  & \; \nabla {L_i}^\top \vv \leq \alpha, \quad i \in [m] . \\
\ee
We would like to mention a less explored way \cite{fliege2000steepest} that also simultaneously reduces all objectives, which serves as a foundation of the proposed PMGDA method. In this way, to find a valid direction $\vd$, we solve the following problem:
\bee \label{eqn:mgda2}
    & (\vd, \alpha^*)= \argmin_{(\vv, \alpha) \in \mathbb{R}^{m+1}} \alpha, \\
    & \st \; \left\{
    \begin{aligned}
        & \nabla {L_i}^\top \vv \leq \alpha, \quad i \in [m], \\
        & \norm{\vv}_{\infty} \leq 1. \\
    \end{aligned}
    \right.
\ee
If the optimal $\alpha^*<0$, then $\vd$ is a direction that can reduce all the objectives, and each objective $L_i$ can be decreased by about $\eta \alpha^*$. When $\alpha^*=0$,  no direction can optimize all the objectives simultaneously, and thus $\vth_k$ is Pareto stationary \cite{fliege2000steepest}.

\section{The Preference-based Multi-Gradient Descent Algorithm (PMGDA)}
In this section, we introduce the Preference-based MGDA (PMGDA) method. Initially, we define several constraint functions $h(\vth)$ in \Cref{sec:pfunc} to meet specific user requirements. Next, \Cref{sec:frame} outlines our conceptual framework for addressing this problem. Detailed methodologies are further elaborated in \Cref{sec:predict}, \ref{sec:correct}, and \ref{sec:reuse}. Subsequently, \Cref{sec:analysis} provides an analysis of the complexity involved, and \Cref{sec:morl} discusses the extension of PMGDA to complex multi-objective reinforcement learning tasks.

\label{sec_method}
\subsection{Preference Functions} \label{sec:pfunc}
\sssec{The exact preference function} 
There are several functions available to align a Pareto objective vector $\vL(\vth)$ with a preference vector $\vlam$. One commonly used function is the Euclidean distance from a Pareto objective to this preference vector \cite{zhang2007moea}. 
\begin{equation} \label{eqn:h_pbi}
    h(\vth) = \norm{\vL(\vth) - \frac{\vL(\vth) \cdot \vlam}{\norm{\vlam}} \vlam}
\end{equation}

This function indicates a solution is exact if and only if $h(\vth)$ is equal to zero.

\sssec{The Region of Interest (ROI) preference function}

ROI is a specific region defined by the decision maker. In addition to the regular ROI discussed in \Cref{sec:moobasic}, we also consider the ROI as a spherical shape. The preference function for the ROI constraint is defined as follows:

\begin{equation} \label{eqn:h_ro}
h(\vth) =
\begin{cases}
0, & \text{if } \vL(\vth) \in \texttt{ROI} \\
\text{dist}(\vL(\vth), \texttt{ROI}), & \text{otherwise}
\end{cases}
\end{equation}
Here, $\text{dist}(\vL(\vth),\texttt{ROI})$ represents the Euclidean distance to the $\texttt{ROI}$ set. $h(\vth) =0$ means that this solution satisfies the ROI constraint. 

\subsection{The Predict-and-Correct Framework}  \label{sec:frame}
The objective of the preference-based MOO (MOO) problem is to identify a Pareto solution that fulfills the preference constraint \(h(\vth) = 0\). Drawing inspiration from Gebken et al.'s work \cite{gebken2017descent}, we propose a predict-and-correct algorithmic framework. In this framework, let \(\vth_k\) represent the solution at iteration \(k\), with the subsequent solution \(\vth_{k+1}\) determined as follows:

\begin{itemize}
\item When $h(\vth_k) = 0$, we do the prediction step, which reduces all objective values $L_i$ without hurting the preference constraint $h(\vth_k)$: 
\begin{equation}  \label{eqn:predict}
    \begin{split}
    (\vd, \alpha^*)= & \argmin_{(\vv, \alpha) \in \mathbb{R}^{m+1}} \alpha + \frac{1}{2}  \norm{ \vv }^2, \\
    \st  & \; \begin{cases}
        \nabla {L_i}^\top \vv \leq \alpha, \quad i \in [m], \\
        \nabla h ^\top \vv = 0.
    \end{cases}  
    \end{split}
\end{equation}
\item Otherwise, we do the correction step. The major requirement is to reduce $h(\vth_{k})$ while the objective values $L_i$ should not be deteriorated too much. Specifically, we compute the direction $\mtx d$ as follows: 
\begin{equation} \label{eqn:correct}
    \begin{split}
        & (\vd, \alpha^*) = \argmin_{(\vv \in \R^n , \alpha \in \R)} \alpha  \\
        & \st \left\{
        \begin{aligned}
            & \g L_i^\top \vv \leq \alpha, \quad i \in [m]   \\
            & \g h^\top \vv \leq -\sigma \norm{\g h } \cdot \norm{\vv}, \\
            & 0 < \norm{\vv} \leq 1.  \\
        \end{aligned}
        \right.			
    \end{split}
\end{equation}
where $0<\sigma<1$ is a preset control parameter. To prevent $\vv$ from becoming a zero vector, we need to constrain $\norm{\vv} > 0$. 
\item Finally, update the decision variable
\begin{equation}
    \label{eqn:update}
    \vth_{k+1}=\vth_k + \eta \vd.
\end{equation}
\end{itemize}
We would like to offer some remarks regarding the conceptual algorithm we just introduced. During the prediction phase, there are two critical constraints to consider. The first one ensures that each objective, denoted as $L_i$, can be reduced by a minimum of $\alpha$. The second constraint is designed to guide the search within the tangent space, ensuring that the resulting $h(\vth_k)$ value remains minimal. It's important to note that the Problem (\ref{eqn:predict}) is essentially a variation of Problem (\ref{eqn:mgda1}) and (\ref{eqn:mgda2}) with an additional quality constraint $\nabla h ^\top \vv = 0$. Therefore, the techniques used for analyzing and solving these equations can also be effectively applied to Problem (\ref{eqn:predict}).

In the correction phase, the first constraint outlined in Problem (\ref{eqn:correct}) aligns with the well-established Armijo rule (Armijo, 1966)~\cite{armijo1966minimization}. This constraint is designed to ensure a sufficient reduction in the constraint function $h(\vth)$. Such a reduction is pivotal for steering $h(\vth)$ towards a stationary solution, a concept validated by Bertsekas in 2016 \cite{bertsekas2016nonlinear}. The second constraint regulates the increment of each $L_i$ during the search process. The third constraint is critical in establishing the search boundary, ensuring that the vector $\vv$ remains non-zero. When the correction step is exactly a retraction map as defined in ~\cite{absil2012projection} onto the manifold $h(\vth)=0$ at $\vth_k$, our method is a special case of the constrained multiple gradient method introduced in ~\cite{gebken2017descent} and converges to a Pareto stationary solution satisfying the exact constraint.

Both Problem \eqref{eqn:predict} and \eqref{eqn:correct} have $n+1$ decision variables, where $n$ represents the number of parameters in a neural network. Given that neural networks typically have a large number of parameters, $n$ is consequently large. Additionally, Problem \eqref{eqn:correct} involves nonlinear constraints. To tackle these challenges, in the following sections, we will develop algorithms specifically designed to efficiently manage the large parameter count, $n$, in both the prediction and projection steps of the solution process.  


Our proposed algorithm differs from Gebken et al.~\cite{gebken2017descent}, which is primarily conceptual and not directly suited for large-scale neural networks. Their assumption of an available projection operator (\(\vth \leftarrow \text{proj}(\vth - \eta \vd)\)) is unrealistic in the context of neural networks. We have significantly modified this correction step in Problem \eqref{eqn:correct} to enhance learning and investigated the prediction algorithm in its dual form. Consequently, our correction and prediction steps involve only $O(m)$ decision variables, compared to $O(n)$ in Gebken et al.~\cite{gebken2017descent}'s original work, where $m$ is the number of objectives and $n$ is the number of decision variables. Given that $n \gg m$, our approach is the first practical method for finding preference-based Pareto solutions in neural networks.

\subsection{An Efficient Prediction Algorithm} \label{sec:predict}
To solve problem (\Cref{eqn:predict}) more efficiently, we explore its dual form. We define the Lagrangian function as \(\mathcal{L}(\alpha, \vv, \vmu, \nu) = \alpha + \frac{1}{2} \|\vv\|^2 + \sum_{i=1}^m \mu_i (\nabla {L_i}^\top \vv - \alpha) + \nu (\nabla h^\top \vv)\), where \(\mu_i \geq 0\). Setting the derivative \(\frac{\partial \mathcal{L}}{\partial \vv} = \mtx 0 \) leads to \(\vv=-\sum_{i=1}^m \mu_i \nabla L_i - \nu \nabla h\). Additionally, taking \(\frac{\partial \mathcal{L}}{\partial \alpha}=0\) results in \(\vmu \in \Delta^{m-1}\). Consequently, the dual problem of \Cref{eqn:predict} is formulated as:

\begin{equation} \label{eqn:predict:dual}
    \min_{(\vmu, \nu) \in \Delta^{m-1} \times \mathbb{R}} \norm{ \sum_{i=1}^m \mu_i \nabla L_i + \nu \nabla h }.
\end{equation}

The resulting optimal direction \(\vd = -\sum_{i=1}^m \mu_i^* \nabla L_i - \nu^* \nabla h\), where \(\mu_i^*\) and \(\nu^*\) are solutions to problem of \Eqref{eqn:predict:dual}. It is noteworthy that \Cref{eqn:predict:dual} only involves \(m+1\) decision variables and constitutes a semi-positive definite programming problem, which can be solved very efficiently compared with it original problem (\Cref{eqn:predict}). 

\subsection{An Efficient Correction Algorithm} \label{sec:correct}
The original correction problem (\Eqref{eqn:correct}) has a large number of decision variables $(n)$. $n$ is particularly large for a neural network problem. Additionally, the last two constraints are nonlinear, which makes the problem challenging to optimize.
To make the problem tractable, we assume the search direction $\vd$ is a linear combination of gradients of the objective functions and the constraint function $h(\vth_k)$. Thus we can add the following condition to Problem (\ref{eqn:correct}),
\begin{equation} \label{eqn:v_compose}
    \vd = -\sum_{i=1}^{m} \mu_i \widehat{\g L_i} - \mu_{m+1} \widehat{\g h}.
\end{equation}
Here, the notation $\hat{(\cdot)}$ is used to denote vector normalization, where $\hat{\vx}=\frac{\vx}{\|\vx\|}$. 
Without loss of generality, we can assume that $(\mu_1,\ldots,\mu_{m+1}) \in \Delta^m$, as we can adjust the learning rate $\eta$ in Equation \eqref{eqn:update} to scale the norm of updating direction $\vd$. 
\begin{lem} \label{prop:cons}
    When $\vth_{k}$ is not Pareto stationary, $0 < \norm{\vv} \leq 1$.    
\end{lem}
\begin{proof}
We first prove that $\norm{\vv} \leq 1$, since the following inequalities hold, 
$$
    \begin{aligned}
        \norm{\sum_{i=1}^{m+1} \mu_i v_i}  \leq \sum_{i=1}^{m+1} \norm{\mu_i v_i}  \leq \sum_{i=1}^{m+1} |\mu_i| = 1,
    \end{aligned}
$$
when $v_i$ is a unit vector and $\mu \in \Delta_{m+1}$. 
We then show that when $\vth_k$ is not Pareto stationary, there does not exist constants $\mu_1, \ldots, \mu_{m+1}$ such that $\vv = \sum_{i=1}^m \mu_i \widehat{\g L_i}+ \mu_{m+1} \widehat{\g h} = \mtx 0$, i.e., $\norm{\vv}>0$.
\end{proof}

Lemma \ref{prop:cons} implies that the last constraint in Problem \eqref{eqn:correct} always holds. Under the substitution \Cref{eqn:v_compose}, the first constraint in Problem \eqref{eqn:correct} becomes,
\begin{equation}
    \begin{aligned}
        \g h^\top \vv \leq - \sigma \norm{\g h} \cdot \norm{\vv}. 
    \end{aligned}
\end{equation}
Again, by Lemma \ref{prop:cons} ($\norm{\vv} \leq 1$), the previous inequality can be further simplified as:
\begin{equation} \label{eqn:constr1}
    \g h^\top \vv \leq -\sigma \norm{\g h}.
\end{equation}

\begin{figure}[h]
    \centering
    \ig[width=0.4\tw]{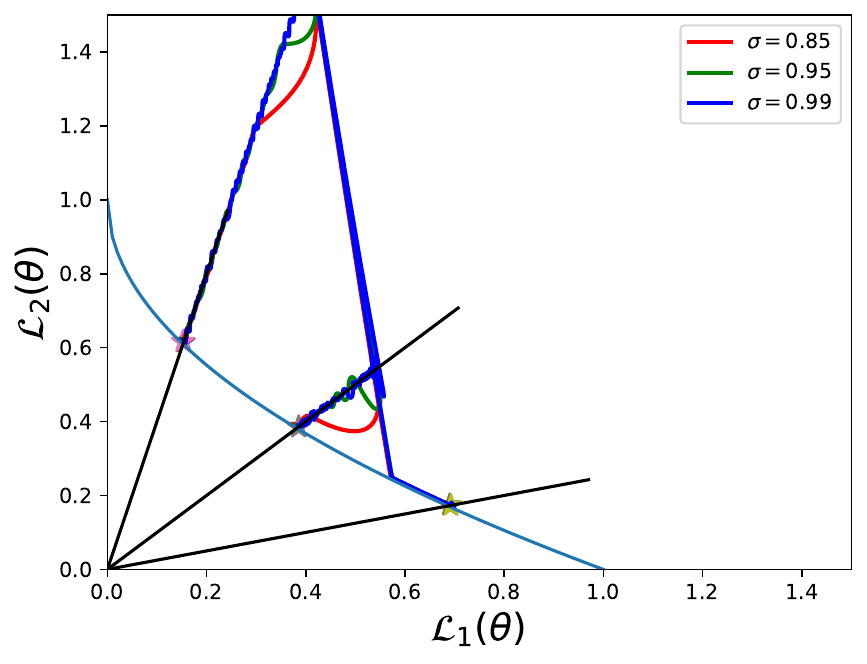}
    \caption{The optimization trajectory of $\sigma=0.85, 0.95, 0.99$. A small setting of $\sigma$ leads to a smooth optimization trajectory, while a large set of $\sigma$ leads to oscillations and a slow convergence rate. Exact Pareto solutions are found separately in 127, 144, and 305 iterations. Marker $\star$ denotes the final solutions. } 
    \label{fig:traj}
\end{figure}
Let $\vG$ (shape is $n \times (m+1)$) be the Jacobian matrix of the function $[\vL(\vth), h(\vth)]$. $\vG$ is formed by concatenating the columns $\nabla L_1$, $\nabla L_2$, ..., $\nabla L_m$, and $\nabla h$.
$$
    \vG = [\nabla L_1, \ldots, \nabla L_m, \nabla h].
$$
Let $\hat{\vG}$ be the normalized augmented Jacobian matrix, where each column is normalized. Then, with \Eqref{eqn:constr1} and Lemma \ref{prop:cons}, the original optimization problem in the correction step (\Eqref{eqn:correct}) can be simplified to the following easy-to-solve linear problem:
\begin{equation} \label{eqn:project:dual}
    \begin{aligned}	
        (\vmu^*, \alpha^*) = & \argmin_{(\vmu, \alpha) \in (\Delta^m \times \R)} \alpha, \\
        s.t.  &  \; (-\vG^\top \hat{\vG}) \vmu \leq \vb,
    \end{aligned}
\end{equation}
where $\vb$ is the vector for the linear constraint which is defined as $\vb = [ \underbrace{\alpha,...,\alpha}_{m+1}, -\sigma \| \g h \| ]^\top$.
We use $\tilde{\vmu}$ to denote the final decision variable, $\tilde{\vmu} = [\vmu, \alpha]$. In this way, Problem \eqref{eqn:project:dual} can be written in a compact matrix form that can be directly solved using the `cvxpy' library \cite{diamond2016cvxpy}. The formulation is as follows:
\begin{equation} \label{eqn:correct:mtx}
    \begin{aligned}	
        \tilde{\vmu}^* = & \argmin_{\tilde{\vmu} \in (\Delta^m \times \R)} \langle \; [\underbrace{0,\ldots,0}_{m+1},1], \tilde{\vmu} \; \rangle, \\
        \st  &  \; 
        \begin{bmatrix}
        \multirow{4}[0]{*}{\centering $-\vG^\top \hat{\vG}$} & -1 \\
         & \vdots \\
         & -1 \\ 
         & 0 \\
        \end{bmatrix} 
        \begin{bmatrix}
            \mu_1 \\ \vdots \\ \mu_{m+1} \\ \alpha
        \end{bmatrix}
        \leq 
        \begin{bmatrix}
            0 \\ \vdots \\ 0 \\ -\sigma \norm{\g h}
        \end{bmatrix}.
    \end{aligned}
\end{equation}
The shape of constraint matrix is $(m+1) \times (m+2)$, while the constraint vector is $(m+1)$. 

Problem (\ref{eqn:correct:mtx}) is a linear problem defined on a compact constraint set, which always has an optimal solution. After solving Problem (\ref{eqn:correct:mtx}), with a running time of $\mathcal{O}^*((2m+3)^{2.37})$ \footnote{$\mathcal{O}^*$ is used to hide the $m^{o(1)}$ and the $\log^{O(1)}(1/\delta)$ factors, details see \cite{pmlr-v119-mahapatra20a,cohen2021solving}.} as proven in \cite{cohen2021solving}, we obtain the optimal solution $\vmu^*$. Consequently, the search direction $\vd$ of \Eqref{eqn:correct} can be set as $\vd = \hat{\vG}^\top \vmu^*$.

The parameter $\sigma$ defines the angle between the search direction $\mtx d$ and the gradient $\g h$. If $\sigma$ is set to 1.0, then $h(\vth_k)$ is optimized by the classic gradient descent, and a small value of $\sigma$ relaxes the angle between the search direction $\mtx d$ and $-\g h$. A large $\sigma$  may make the convergence trajectory to oscillate, resulting in slow convergence. As shown in \Cref{fig:traj}, PMGDA behaves different with different $\sigma$. 

Since Problem \eqref{eqn:correct:mtx} has $(m+2)$ variables, where $m$ is the number of objectives (typically small), our proposed algorithm scales well to neural networks with a large number of parameters ($n$ is large).

\subsection{Reuse of Gradients} \label{sec:reuse}
For simple problems, obtaining the gradient of $\nabla h$ directly is cheap. However, for large-scale problems, such as neural networks, computing gradients becomes expensive. Note that $h(\vth)$ is a function that only involves objectives $L_i$'s. Once all the gradients $\nabla L_i$'s are obtained, the gradient of the constraint function $\nabla h$ can be estimated using the chain rule:

\begin{equation} \label{eqn:chain_rule}
    \nabla h = \sum_{i=1}^{m} \left( \frac{\partial h}{\partial L_i} \cdot \nabla L_i \right),
\end{equation}
\noindent where the first term $\frac{\partial h}{\partial L_i}$ are just scalars.

For the specific case when $m=2$ (two-objective problem) and $h(\cdot)$ is defined by \Eqref{eqn:h_pbi}, the partial derivatives $\frac{\partial h}{\partial L_i}$ can be calculated using the following closed-form:

\begin{equation}
    \begin{cases}
        \frac{\partial h}{\partial L_1} = \left(- \frac{\lambda_2}{\|\vlam\|}, \frac{\lambda_1}{\|\vlam\|} \right), & \text{if } \frac{L_1(\vth)}{\lambda_1} < \frac{L_2(\vth)}{\lambda_2}, \\
        \frac{\partial h}{\partial L_2} = \left( \frac{\lambda_2}{\|\vlam\|}, -\frac{\lambda_1}{\|\vlam\|} \right), & \text{otherwise},
    \end{cases}
\end{equation}

\noindent where $\vlam$ is the $m$-D preference vector. 

For other preference functions $h(\vth)$ which may not have an analytical gradient expression, it is possible to perform a lightweight backward propagation to obtain the values of $\frac{\partial h}{\partial L_i}$.

\parag{PMGDA as a Dynamic Weight Adjustment (DWA)}
We now demonstrate that the proposed PMGDA can be implemented even more efficiently through a dynamic weight adjustment mechanism. The final direction vector \(\vd\) can be succinctly expressed using the following equations, which is only a linear combination of gradients $\g L_i$'s.
\bee \label{eqn:dwa}
    \vd & = - \sbr{ \sum_{i=1}^m \mu_i^* \widehat{\g L_i} + \mu_{m+1}^* \widehat{\g h} } \\
    & \text{(Substitute the chain rule (\Eqref{eqn:chain_rule}) to the above equation.)} \\
    & = - \sbr{ \sum_{i=1}^m \frac{\mu_i^*}{\norm{\g L_i}} \g L_i + \frac{\mu_{m+1}^*}{\norm{\g h}} \sum_{i=1}^m \frac{\partial h}{\partial L_i} \g L_i} \\
    & = - \sbr{ \sum_{i=1}^m \sbr{\frac{\vmu^*_i}{\norm{\g L_i}} + \frac{\mu_{m+1}^*}{\norm{\g h}} \frac{\partial h}{\partial L_i}} \g L_i } . \\
    & = - \sbr{ \sum_{i=1}^m \mu_i^\prime \g L_i } \quad \mu_i^\prime := \sbr{ \frac{\mu_i^*}{\norm{\g L_i}} + \frac{\mu_{m+1}^*}{\norm{\g h}} \frac{\partial h}{\partial L_i}}. \\
\ee

\subsection{Algorithm and its Complexity Analysis} \label{sec:analysis}
\begin{algorithm}[h!]
    \caption{The Preference based Multiple Gradient Descent Algorithm (PMGDA) } \label{alg:dual}
    \textbf{Input:} A user preference $\vlam$, initial parameter $\vth_0$, iteration round $k=0$. \\
    \While {Stopping condition is not met}{
        \eIf{ $h(\vth_{k}) < \delta$ ($\delta$ is a small margin (e.g., $10^{-2}$). ) } 
        { 
        // Do the prediction step. \\

        Find the direction $\vd$ by solving the prediction MGDA problem (\Eqref{eqn:predict:dual}). 
        } 
        {
        // Do the correction step. \\
        Solve the coefficient $\vmu^*$ in linear programming problem (\Eqref{eqn:correct:mtx}). \\
        Calculate the dynamic weight $\vmu'$ by \Eqref{eqn:dwa}. \\
        $\vd = -\sum_{i=1}^m \mu_i^\prime \nabla L_i$. 
        }
        $\vth_{k+1} = \vth_k + \eta \vd.$ \\ 
        $k \leftarrow k+1$. \\
    }
    \textbf{Output: The final solution $\vth_k$}
\end{algorithm}

We summarize our practical method as \Cref{alg:dual} and call it PMGDA. The stopping condition used in our experiments is that $\|d\| \leq \epsilon$ and $h(\vth_k) \leq \delta$. This condition is met when $d=0$ and $h(\vth_k)=0$, indicating that the solution meets decision maker's demand.

The main time consumption of PMGDA in the `correction' step is the calculation of $\vG^\top \hat{\vG}$, which has a complexity of $\mathcal{O}(({m+1})^2 n)$. This calculation time dominates the optimization time, which has a $\mO^*({(2m+3)}^{2.37})$ complexity. According to out empirical results, this step is converged in only several iterations.
In the `prediction' step, the running time of Pure MGDA \cite{fliege2000steepest,desideri2012mutiple} is $\mathcal{O}(m^2 n)$. It is worth noting that MGDA does not guarantee the position its searched Pareto solution. On the other hand, it is advantageous our method can generate a specific solution satisfying a decision maker's demand without significantly increasing complexity. The complexity of the number of variables is linear, while the complexity with respect to the objective is quadratic.

\subsection{Multi-objective Reinforcement Learning (MORL) Task} \label{sec:morl}
In this section, we demonstrate the adaptability of our proposed framework to MORL tasks, which are very practical real-world problems. In MORL, each state-action pair \((s, a)\) is associated with \(m\) conflicting rewards, denoted as \(r_1(s,a), \ldots, r_m(s,a)\). The objective in MORL is to optimize \(m\) cumulative returns, each defined by:

\begin{equation}
    L_i(\vth) = \mathbb{E}_{\pi \sim \pi_{\vth}(\vlam)} \left[ \sum_{t=1}^{T} \gamma^\top r_i(s_t, a_t) \right], \forall i \in [m],
\end{equation}

\noindent where \(\pi_{\vth}(\vlam)\) represents a policy network parameterized by \(\vth\), mapping states to actions. Here, \(0<\gamma<1\) is the discount factor, and \(T\) signifies the total time steps. The gradient of \(L_i(\vth)\) can be estimated through interaction with the environment. We illustrate this using the most direct policy gradient method \cite{sutton1999policy}, shown as:

\begin{equation}
\g L_i = \mathbb{E} \left[ \sum_{t=0}^{T} \nabla_{\vth} \log \pi_{\vth}(a_t | s_t) (\mtx G_i(s_t) - \vb_i(s_t)) \right],
\end{equation}

\noindent where \(\mtx G_i(s_t) = \sum_{t'=t}^\top \gamma^{t' - t} r_{t', i}\) represents the discounted cumulative reward from timestep \(t\) to \(T\), and \(\vb_i(s_t)\) is a baseline function. Practically, \(\vb_i(s_t)\) is approximated by a value network \(\vv_{\mtx \phi_i}(s)\), and \(\nabla L_i\) is estimated using \(N\) trajectory samples:

\begin{equation} \label{eq:g_rl}
    \nabla L_i \approx \frac{1}{N} \sum_{i=1}^N  \sum_{t=0}^{T-1} \nabla_{\vth} \log \pi_{\vth}(a_t | s_t) (\mtx G_i(s_t) - \vv_{\mtx \phi_i}(s_t)).
\end{equation}

The integration of PMGDA within the context of MORL introduces two distinct differences. Firstly, the computation of the objectives \(L_i\) is based on evaluation metrics, and their gradients \(\nabla L_i\) are derived from a separate evaluation problem as defined in Equation \eqref{eq:g_rl} \footnote{\(\nabla L_i\) can also be estimated by the PPO \cite{schulman2017proximal} or TRPO \cite{schulman2015trust} algorithms to improve sample efficiency}. The gradient $\nabla h(\vth)$ can only calculated with a chain rule of estimated gradients, $(\frac{\partial h}{\partial L_i})$'s and $\g L_i$'s.

Secondly, MORL involves maximization rather than minimization, hence, the constraint formulation needs to be adjusted accordingly. The constraint \(\nabla L_i^\top \vv \leq \alpha, \; i \in [m]\) should be altered to a `greater than' condition in Equations \eqref{eqn:predict} and \eqref{eqn:correct}.

\section{Experiments} \label{sec_exp}
We conducted experiments on three types of problems: synthetic problems, multitask fairness classification problems, and multiobjective reinforcement learning problems. 
For synthetic problems, the decision vector $\vth$ is a 30-D vector. For the other two problems, the decision variables are the parameters of neural networks (this number is around 5000-10000). 

\parag{Metrics} We evaluate the quality of a single solution using the PBI indicator \cite{zhang2007moea}, and the cross angle $\vartheta$ between the objective vector and the preference. We also use the hypervolume (HV) indicator \cite{guerreiro2020hypervolume} to evaluate the convergence of a set of solutions. 

\begin{enumerate}
    \item The cross angle indicator $\vartheta = \text{arccos} \frac{\vL^\top \vlam}{\norm{\vL} \norm{\vlam}}$. A value of 0 for $\vartheta$ indicates an exact solution.
    
    \item The Penalty Boundary Intersection (PBI) indicator \cite{zhang2007moea},
    $$
        I^{\text{PBI}}(\vL,\vlam) = d_1 + \mu d_2,
    $$ 
    where $d_1$ is the projection distance to preference vector, $d_1=\frac{\vL^\top \vlam}{\norm{\vlam}}$, $d_2$ is the distance to the objective vector $\vL$, and $\mu$ is a weight factor, which is set 8.0. The small PBI indicator means the solution is both close to the given preference vector and the true Pareto front. 
    
    \item The HV indicator measures the quality of a set of objectives $\sA$. 
    \begin{equation}
        \mathcal{H}_\vr(\sA) \defeq \Lambda(\{ \vq \; |\; \exists \vp \in \sA : \vp \preceq \vq \; \text{and} \; \vq \preceq \vr \}),
    \end{equation}
    where $\vr$ is a reference point and $\Lambda(\cdot)$ serves as the Lebesgue measure.  
\end{enumerate}

\parag{Baseline methods}
Our code is current provided in the \texttt{pgmda.rar} attached file and will be open-sourced after publication. 
When searching for exact Pareto solutions, we compared our method with the following baseline methods:
\begin{enumerate}
    \item The Exact Pareto Optimization (EPO) \footnote{Code: \url{https://github.com/dbmptr/EPOSearch}.}\cite{pmlr-v119-mahapatra20a} method. This method to search for exact Pareto solutions before our method. EPO has two steps. The first step is to align the solution with a preference and the second step is to push the solution to the Pareto front. These two steps are repeated until an exact Pareto solution is found. We have fixed numerical bugs in the EPO code. Refer to our code (\texttt{pgmda.rar}) for more details.

    \item The COSMOS \footnote{Code: \url{https://github.com/ruchtem/cosmos}.} \cite{ruchte2021cosmos} method. This method that aims to search for approximate exact Pareto solutions. The COSMOS optimizes the scalarization function,
    $$g^{\text{cosmos}}(\vth)=\vlam^\top \vL(\vth)+\mu S_c(\vL(\vth), \vlam).$$
    This function penalizes the distance to the Pareto front and the cosine similarity $S_c$ with a preference $\vlam$. $\mu$ is a weighting factor, which is set to 5.0 following their work \cite{ruchte2021cosmos}. As the contour of COSMOS shown in \Cref{fig:contour}, when this factor is now well selected, the minimal solution of COSMOS does not imply the `exact' Pareto solution. How to set this factor to balance exactness and convergence remains unknown.  
    \begin{figure}[h]
    \centering
    \ig[width=0.24\tw]{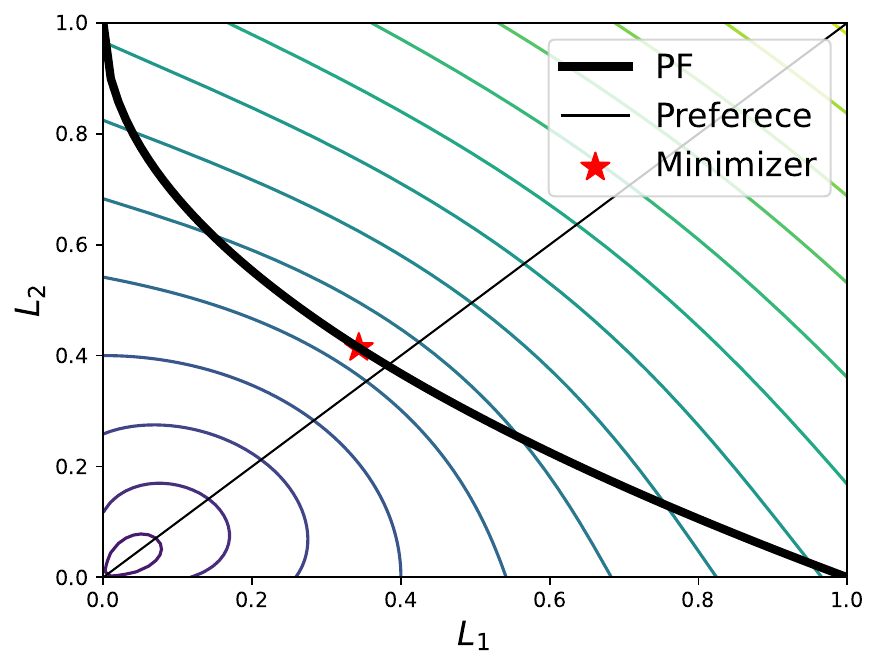}
    \caption{Contour of COSMOS. For a given preference $\vlam$, the minimizer of the COSMOS function does not necessary corresponding to the exact Pareto solution (i.e., the COSMOS contour curve is not tangent to the PF). } 
    \label{fig:contour}
\end{figure}
    \item The modified Tchebycheff (mTche) method. This method aim to minimize the scalarization function,
    $$g^{\text{mTche}}(\vth)= \max_{i \in [m]} \lbr{\frac{L_i(\vth) - z_i}{\lambda_i}},$$ where $\vz$ is a reference point and is typical set as the ideal point of a MOO problem. It is known in \cite{ma2017tchebycheff,zhang2023hypervolume} that under mild conditions, mTche can find the `exact' Pareto solution. 
\end{enumerate}

\subsection{Synthetic Problems} \label{section_synthetic}
In this section, we evaluate the performance of PMGDA on six widely-used gradient-based problems: ZDT1\cite{zitzler2000comparison}, MAF1\cite{cheng2017benchmark}, and DTLZ2 \cite{deb2002scalable}. These problems are commonly employed in the gradient-based MOO literature \cite{lin2019pareto,lin2022paretob,liu2021profiling,pmlr-v119-mahapatra20a,lin2020controllable}.

Initial solutions for VLMOP2 are randomly sampled from the range ${[-0.25, 0.25]}^{30}$. For the other problems, initial solutions are randomly sampled from ${[0, 1]}^{30}$. All methods are run for 500 iterations. For 2-objective problems, 8 preferences are uniformly selected from within the first quadrant, and for 3-objective problems, 15 preferences are uniformly selected from the three-dimensional simplex. The experiments are performed using five different random seeds. The control parameter $\sigma$ in PMGDA is set as 0.9.  

We present the convergence curve of the cross angle, denoted as $\vartheta$, and the hypervolume in Figures \ref{fig:lc:zdt1}, \ref{fig:lc:maf1}, and \ref{fig:lc:dtlz2}. The cross angle reported in these figures corresponds to the maximum cross angle among all preferences. The illusive results on MAF1 is shown in \Cref{fig:res:maf1}. All results are shown in \Cref{tab:res:exact}. 

\begin{table}[]
\caption{Results on synthetic problems. The best solution is marked in bold and the second best is marked with an underline.} \label{tab:res:exact}
\begin{tabular}{llrrrr}
    \toprule
    Problem & Ind. & COSMOS & EPO & mTche & PMGDA \\
    \midrule
    ZDT1 & $\vartheta$ $\downarrow$ & 11.02 & 11.64 & \underline{1.58} & \textbf{1.46} \\
     & PBI $\downarrow$ & 1.08 & 0.82 & \underline{0.73} & \textbf{0.72} \\
    \midrule
    MAF1 & $\vartheta$ & 6.68 & 17.27 & \underline{1.46} & \textbf{1.30} \\
     & PBI & 1.86 & 1.97 & \underline{1.28} & \textbf{1.27} \\
    \midrule
    DTLZ2 & $\vartheta$ & \textbf{1.33} & 33.01 & 28.70 & \underline{1.34} \\
     & PBI & \textbf{1.00} & 1.67 & 1.36 & \underline{1.01} \\
    \bottomrule
\end{tabular}
\end{table}

\begin{figure}%
    \centering
    \sfig[Hypervolume]{\ig[width= \fwdt \tw]{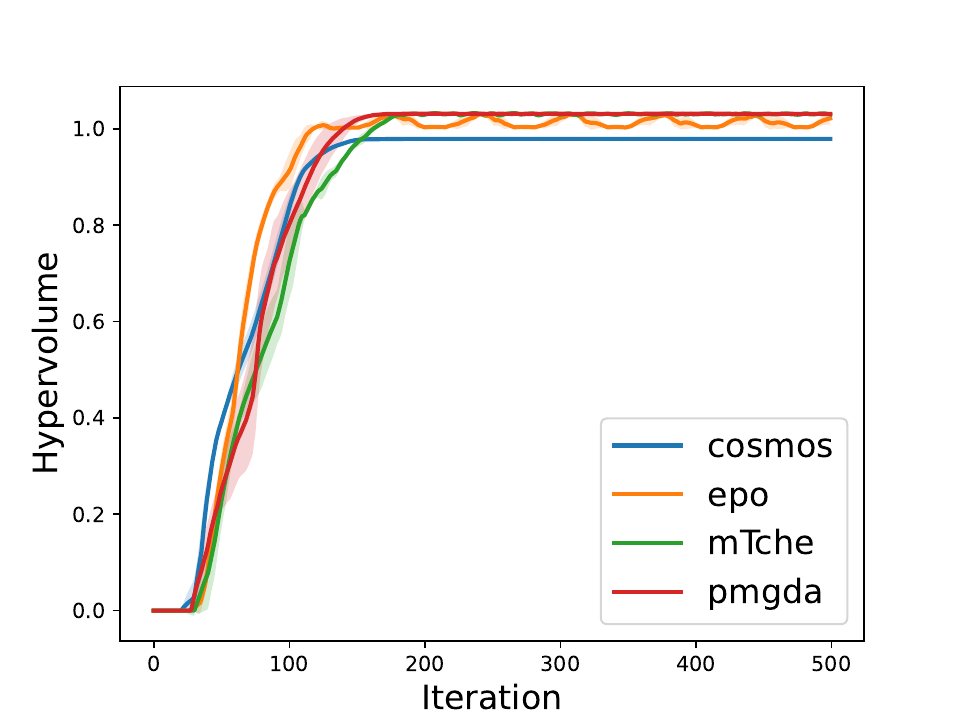}}
    \sfig[Cross angle $\vartheta$]{\ig[width= \fwdt \tw]{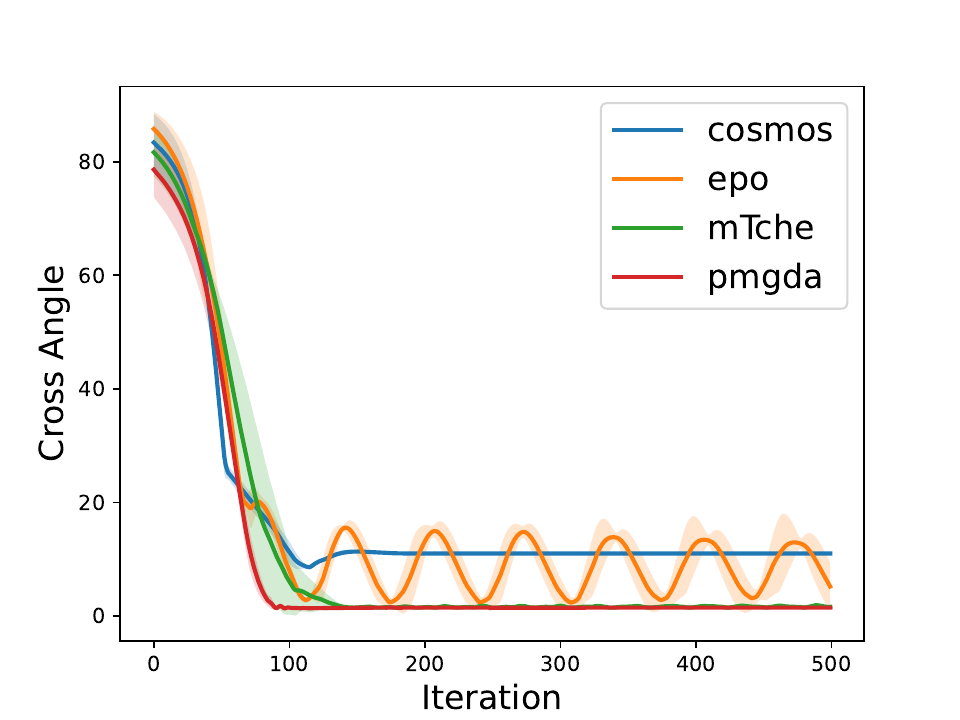}}
    \caption{Learning curves on ZDT1. The proposed method has the best convergence speed and final performance. The convergence curve of EPO fluctuates before convergence. Final solutions of COSMOS are also not exact Pareto solutions.} \label{fig:lc:zdt1}
\end{figure}

\begin{figure}%
    \centering
    \sfig[Hypervolume]{\ig[width= \fwdt \tw]{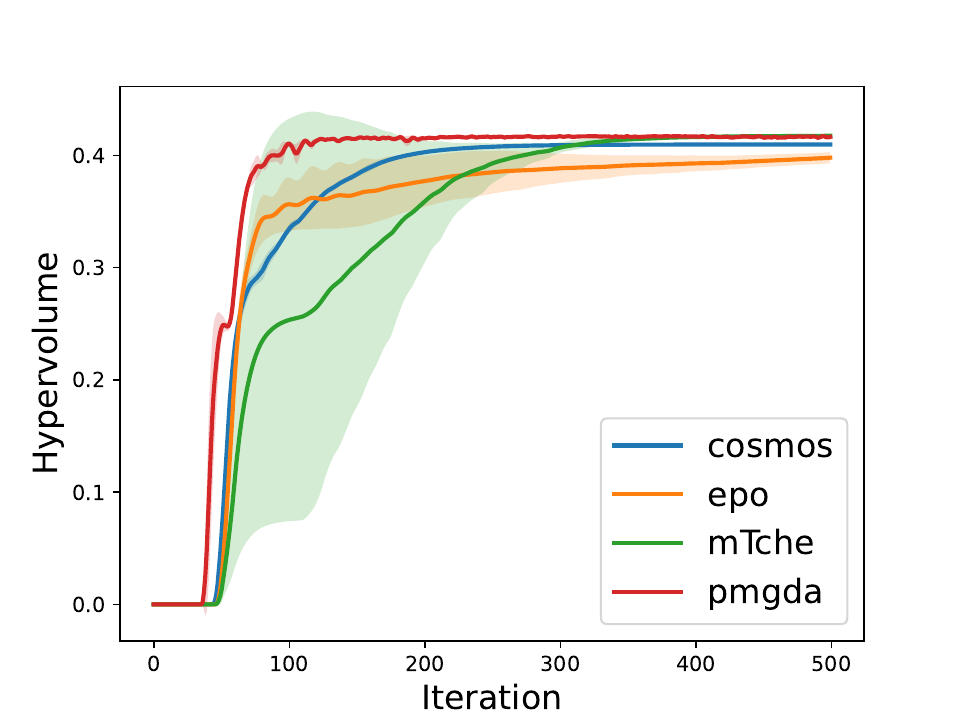}}
    \sfig[Cross angle $\vartheta$]{\ig[width= \fwdt \tw]{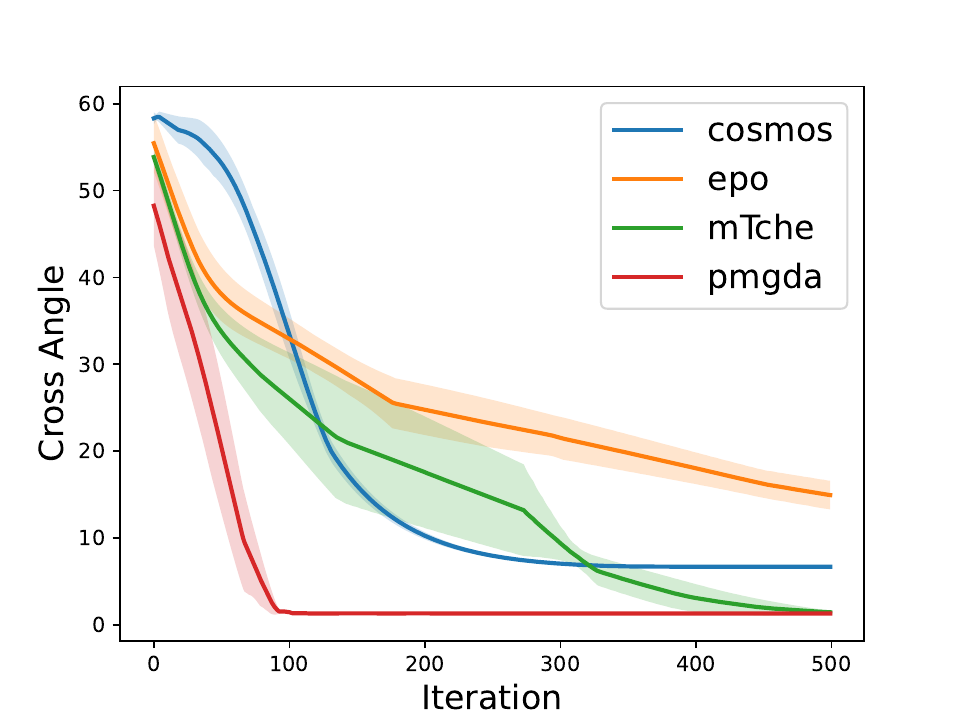}}
    \caption{Learning curves on MAF1. Final solutions by COSMOS and EPO are not exact Pareto solutions. On this three-objective problem, the convergence of mTche is slow, taking 500 iterations to converge, while PMGDA only needs 100 iterations to converge.} \label{fig:lc:maf1}
\end{figure}

\begin{figure}%
    \centering
    \sfig[Hypervolume]{\ig[width= \fwdt \tw]{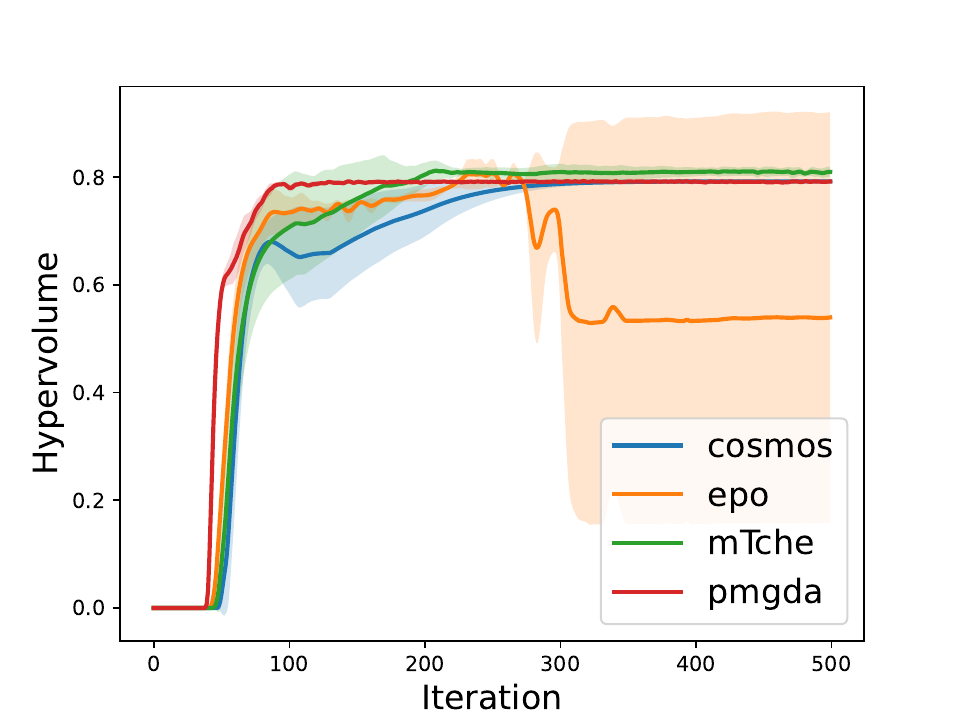}}
    \sfig[Cross angle $\vartheta$]{\ig[width= \fwdt \tw]{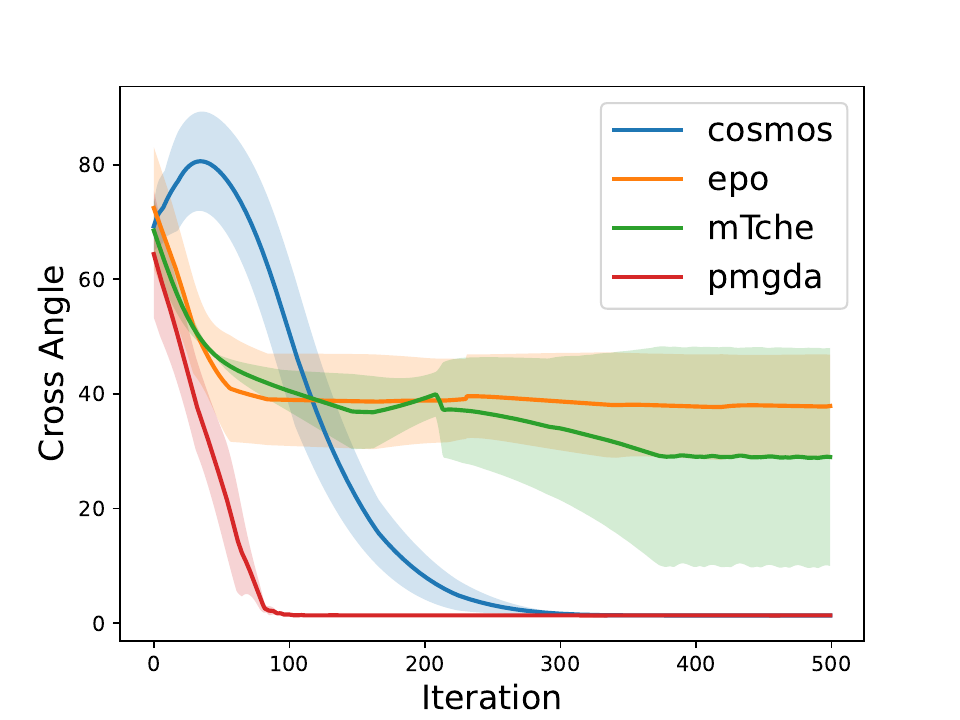}}
    \caption{Learning curves on DTLZ2. PMGDA and COSMOS successfully finds all exact Pareto solutions (PMGDA is 3x faster than COSMOS). However, EPO and mTche fail to find all exact Pareto solutions in 500 iterations. The learning curve of EPO is unstable.} \label{fig:lc:dtlz2}
\end{figure}

\begin{figure}%
    \centering
    \sfig[COSMOS]{\ig[width= \fwdt\tw]{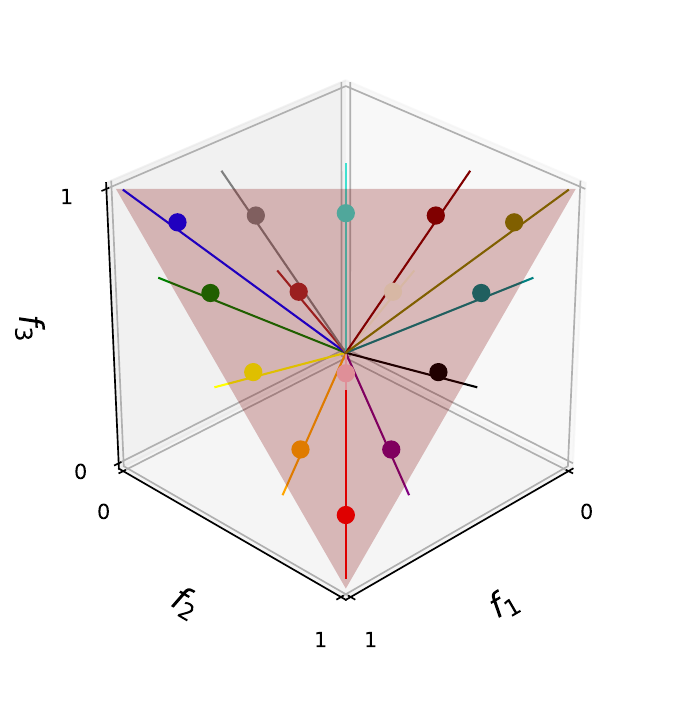}}
    \sfig[EPO]{\ig[width= \fwdt\tw]{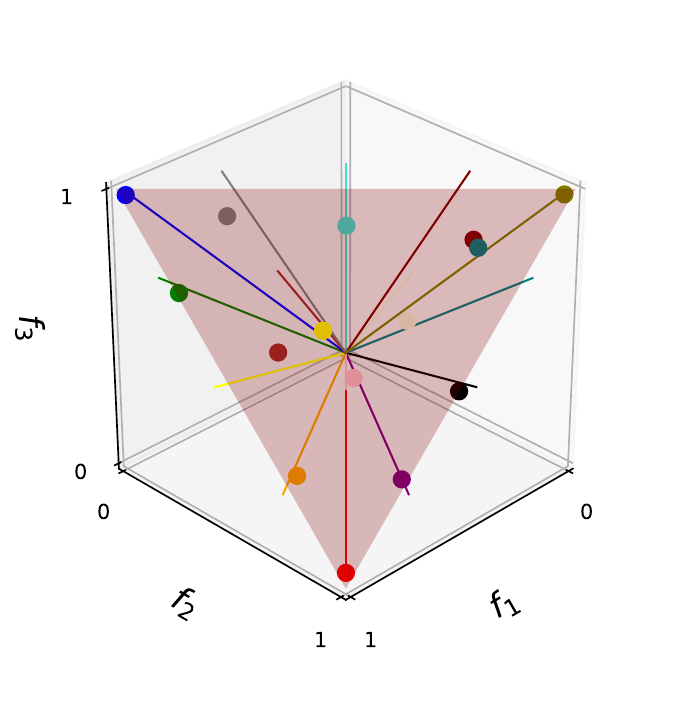}} \\
    \vspace{-10pt}
    \sfig[mTche]{\ig[width= \fwdt \tw]{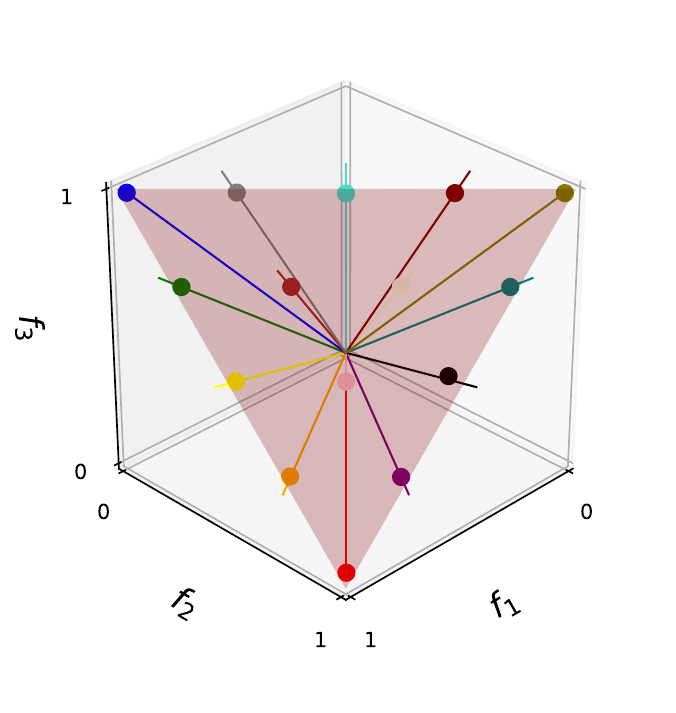}}
    \sfig[PMGDA]{\ig[width= \fwdt \tw]{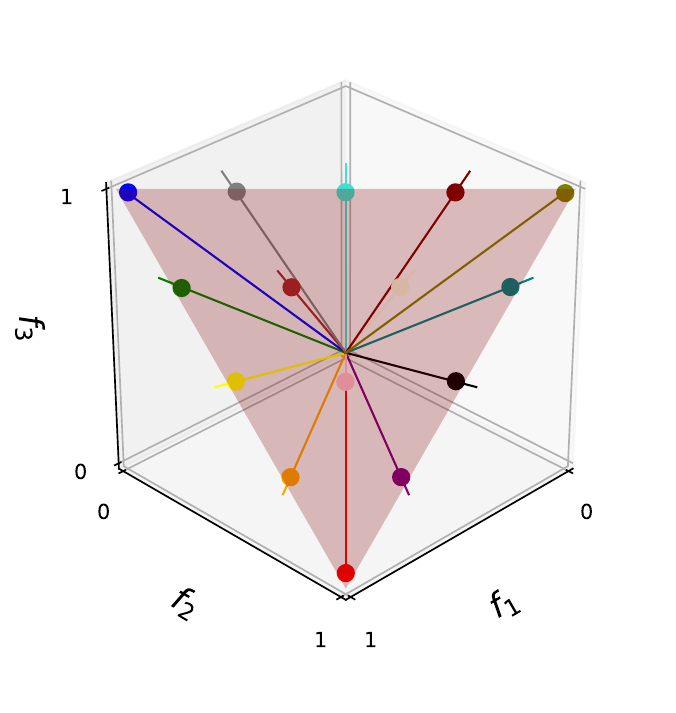}}
    \caption{Illustrative results on the MAF1 problem. Many solutions searched by EPO and COSMOS are not exact Pareto solutions, i.e., solutions are not aligned with preference vectors.} \label{fig:res:maf1}
\end{figure}

The failure of COSMOS is attributed to its difficulty of controlling parameter setting \cite{ruchte2021cosmos}. As shown in \Cref{fig:contour}, even this parameter is set as a large value, the optimizer of the COSMOS objective function does not correspond to the exact solution. 

The original paper suggests a default value of 5.0 for this hyper-parameter. However, COSMOS can only effectively control the average cross angle on the DTLZ2 problem, while it fails to generate precise solutions on the other two problems. Finding the appropriate setting for this hyper-parameter to achieve a balance between convergence, precision with the preference $\vlam$ remains an unresolved challenge.

As shown in \Cref{fig:res:maf1}, EPO cannot produce several `exact' solutions on the MAF1 problem. Additionally, EPO exhibits instability, with fluctuating learning curves on the ZDT1 and DTLZ2 problems, as depicted in \Cref{fig:lc:zdt1} and \Cref{fig:lc:dtlz2}.

When comparing mTche to the proposed method, mTche shows slow convergence due to its strict constraint on the Pareto objectives, which must satisfy the `exact' constraint. In contrast, the proposed method optimizes the constraint function and objectives collaboratively, leading to faster convergence. This is demonstrated in \Cref{fig:lc:maf1} and \Cref{fig:lc:dtlz2}, where the proposed method achieves satisfactory results within 100 iterations, while mTche requires more than 500 iterations.

The efficiency of the proposed method in converging to predefined ROIs is demonstrated in \Cref{fig:res:box}. Notably, the proposed method is the first approach to successfully address this task.
\begin{figure}%
    \centering
    \sfig[Circle constraint]{\ig[width= \fwdt \tw]{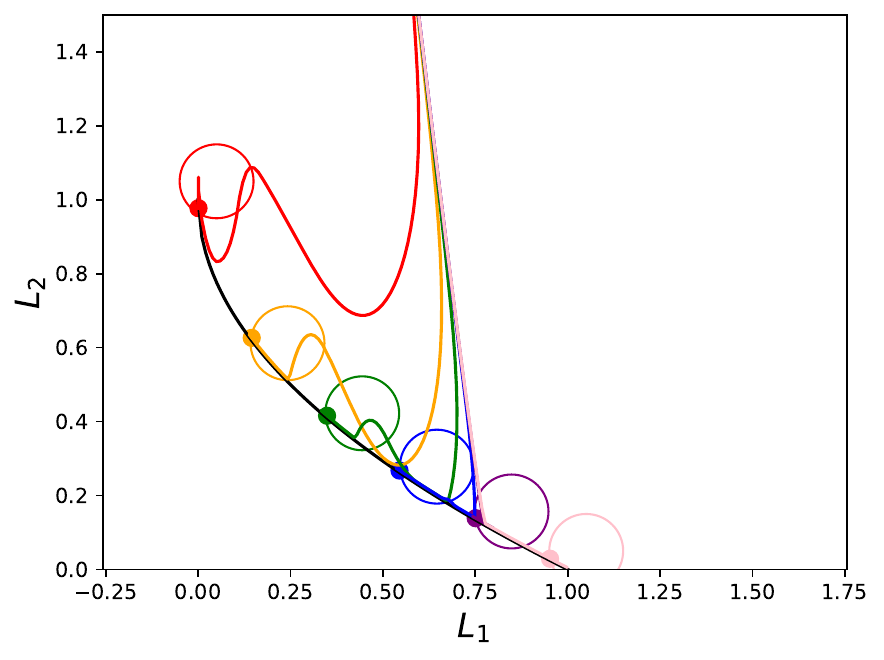}}
    \sfig[Box constraint]{\ig[width= \fwdt \tw]{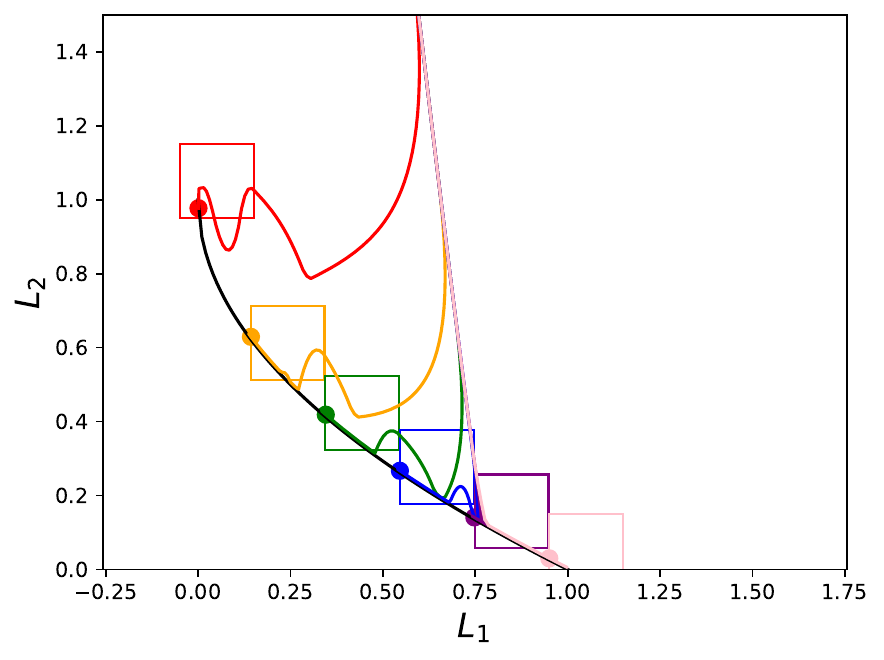}}
    \caption{The optimization trajectory in the ROI task. PMGDA successfully find all Pareto solutions satisfying ROI constraints.} \label{fig:res:box}
\end{figure}

\subsection{Fairness Classification Task}
\begin{figure*}%
    \centering
    \sfig[Adult]{\ig[width= 0.3 \tw]{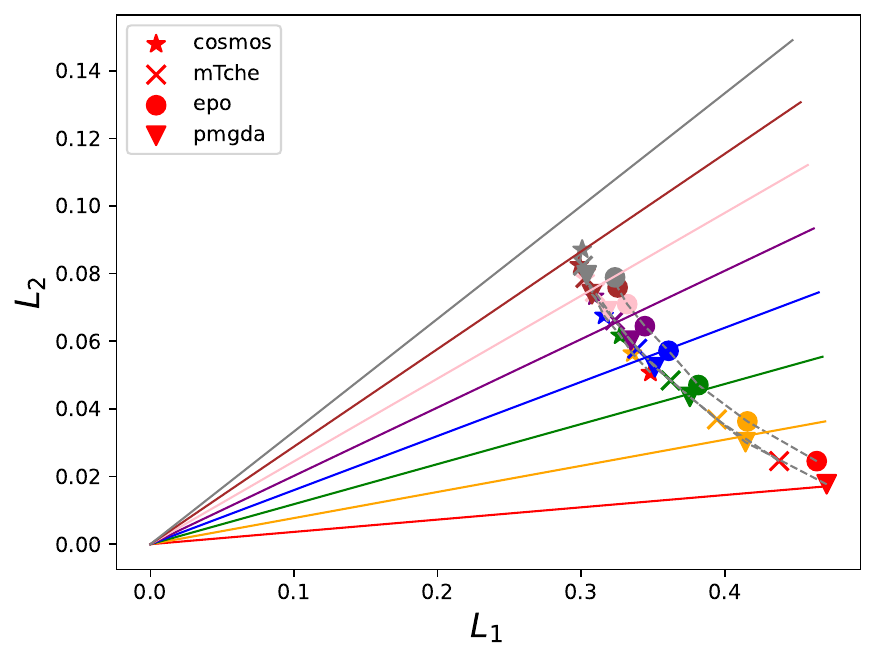}}
    \sfig[Compas]{\ig[width= 0.3 \tw]{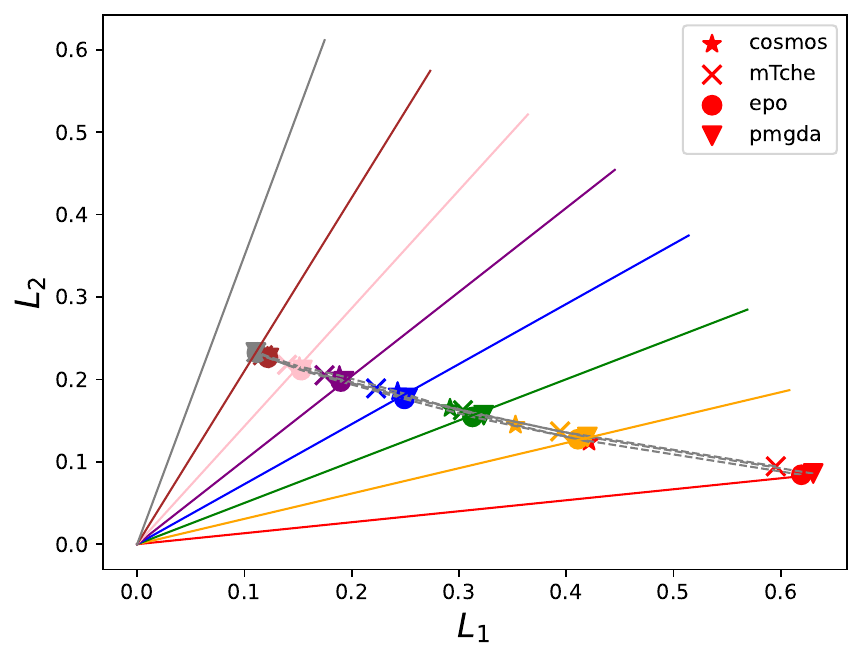}} 
    \sfig[Credit]{\ig[width= 0.3 \tw]{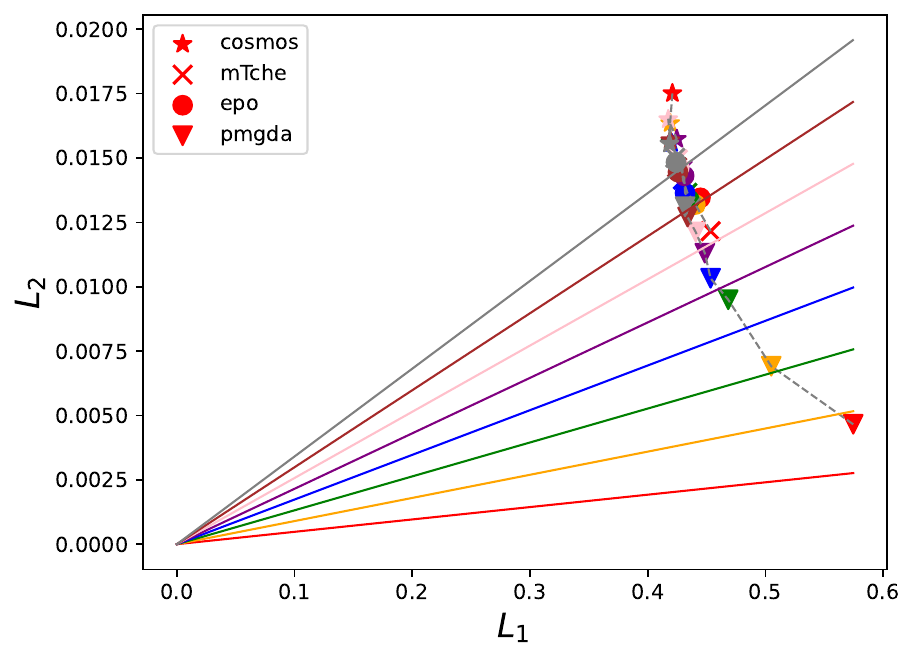}} \\
    \caption{Result comparison by different methods on the fairness classification problems. Generally, PMGDA solutions align better than other methods. PMGDA find a wider PF than other methods on Adult and Credit and converges better than EPO on Adult. On the most simple Compass problem, EPO and PMGDA perform similiarly.} \label{fig:mtl:all}
\end{figure*}

\begin{figure}%
    \centering
    \sfig[EPO results on Adult]{\ig[width= \fwdt \tw]{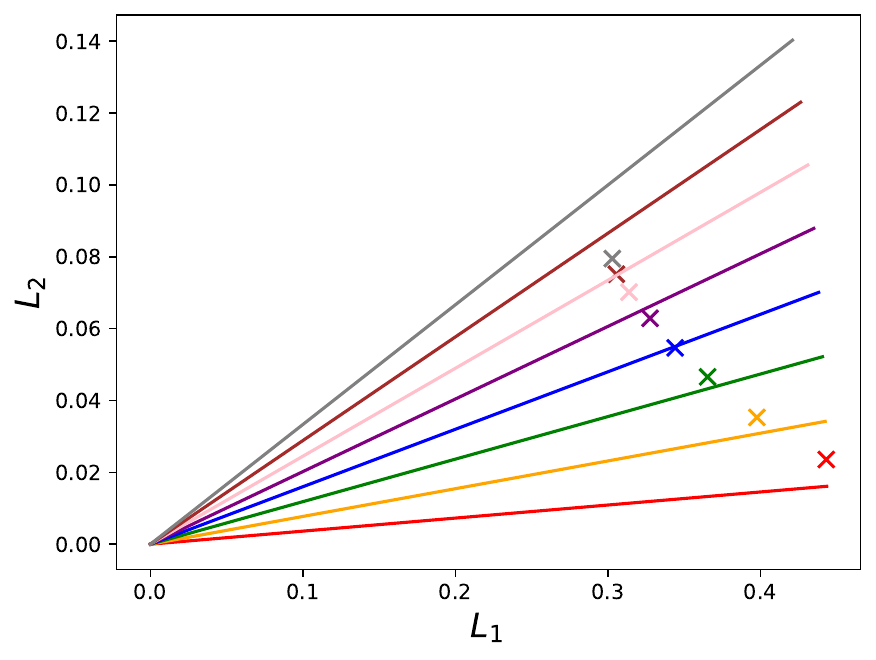}}
    \sfig[PMGDA results on Adult]{\ig[width= \fwdt \tw]{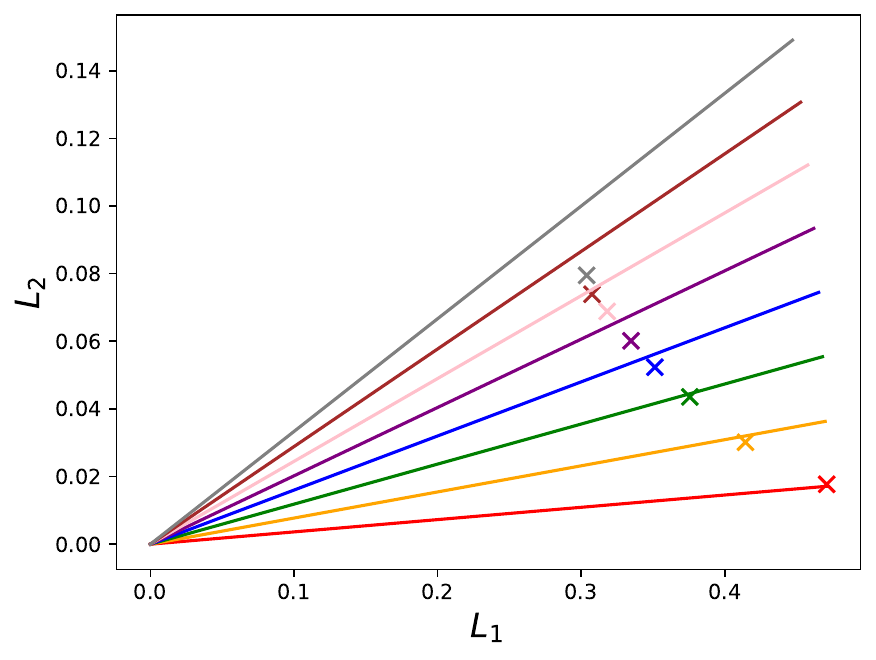}} \\
    \sfig[EPO results on Credit]{\ig[width= \fwdt \tw]{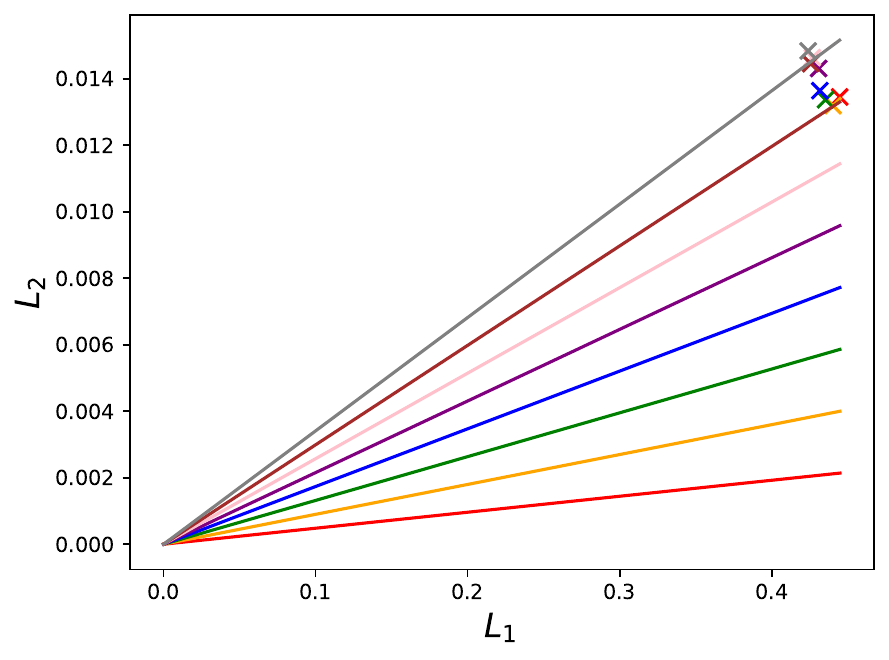}}
    \sfig[PMGDA results on Credit]{\ig[width= \fwdt \tw]{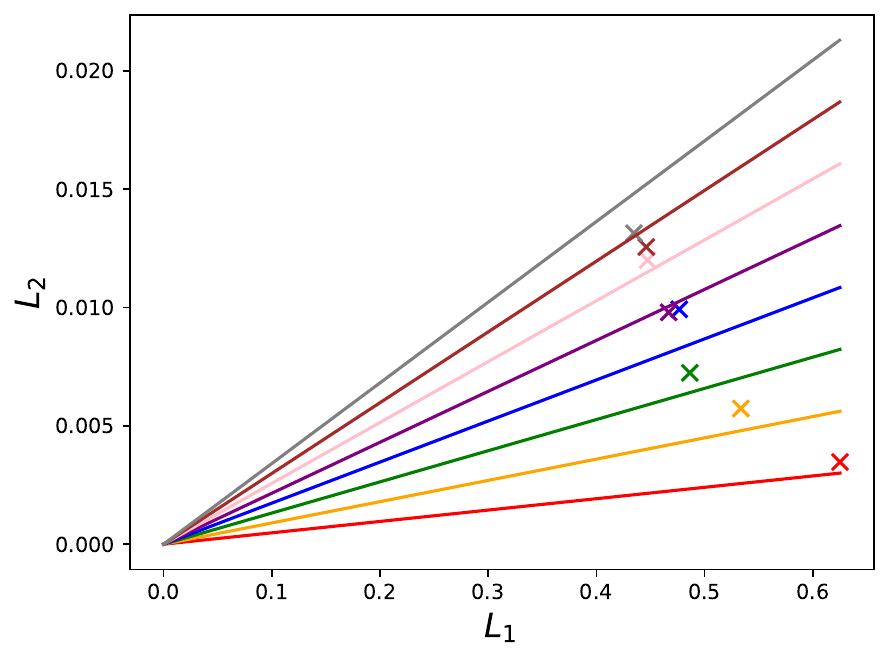}} \\
    \caption{Result comparisons on Adult and Credit. EPO solutions on Credit is only focused on a small region. PMGDA generally find better exact Pareto solutions on most preferences. } \label{fig:mtl:epo}
\end{figure}

\begin{figure}%
    \centering
    \sfig[Adult Circle]{\ig[width= \hwdt \tw]{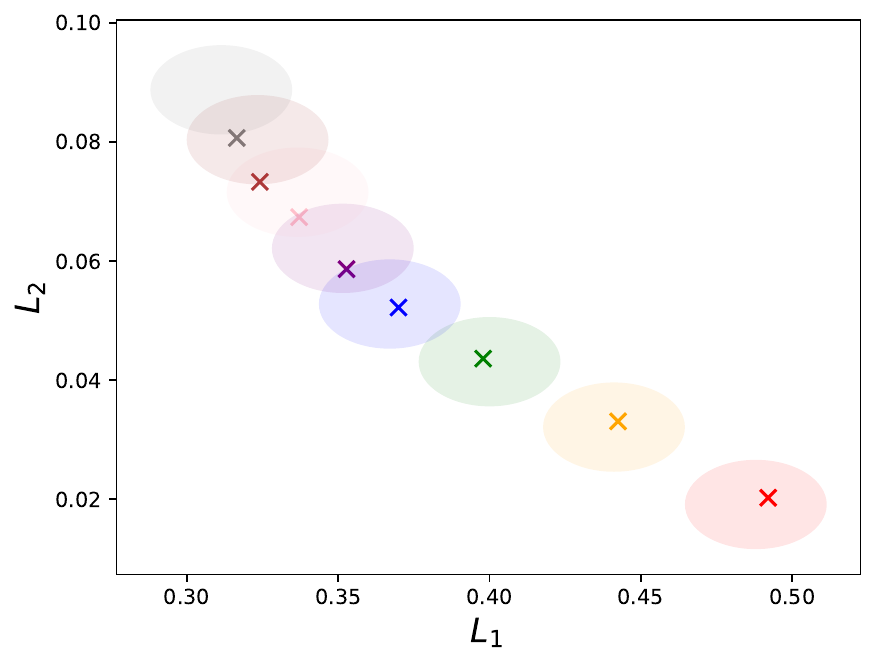}}
    \sfig[Compas Box]{\ig[width= \hwdt \tw]{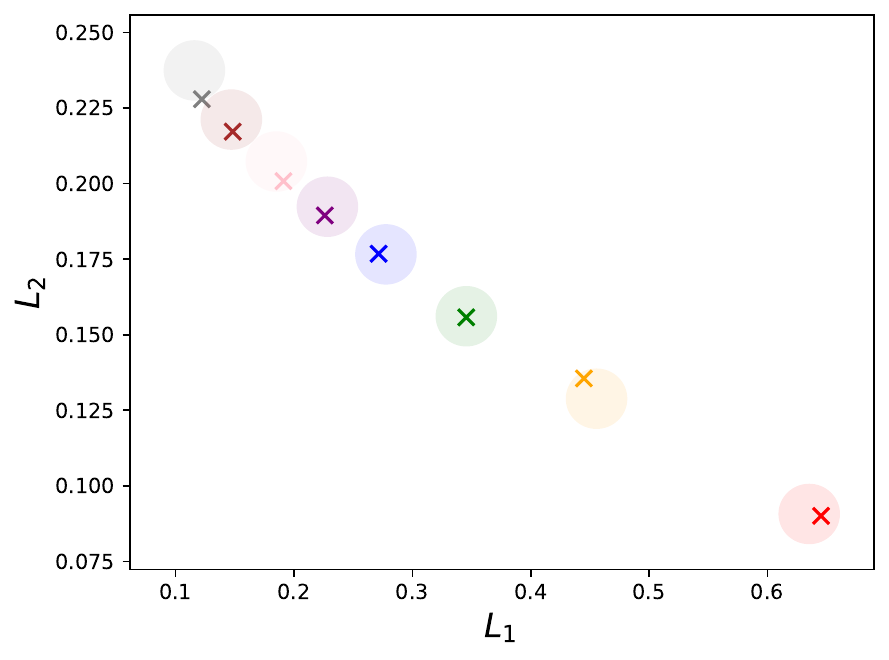}} 
    \sfig[Credit Circle]{\ig[width= \hwdt \tw]{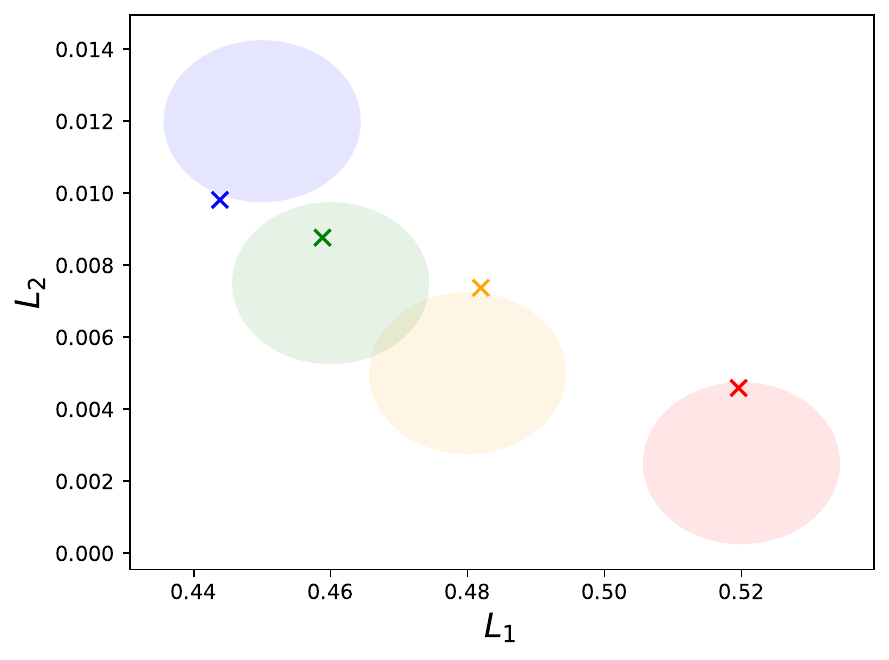}} \\
    \sfig[Adult Box]{\ig[width= \hwdt \tw]{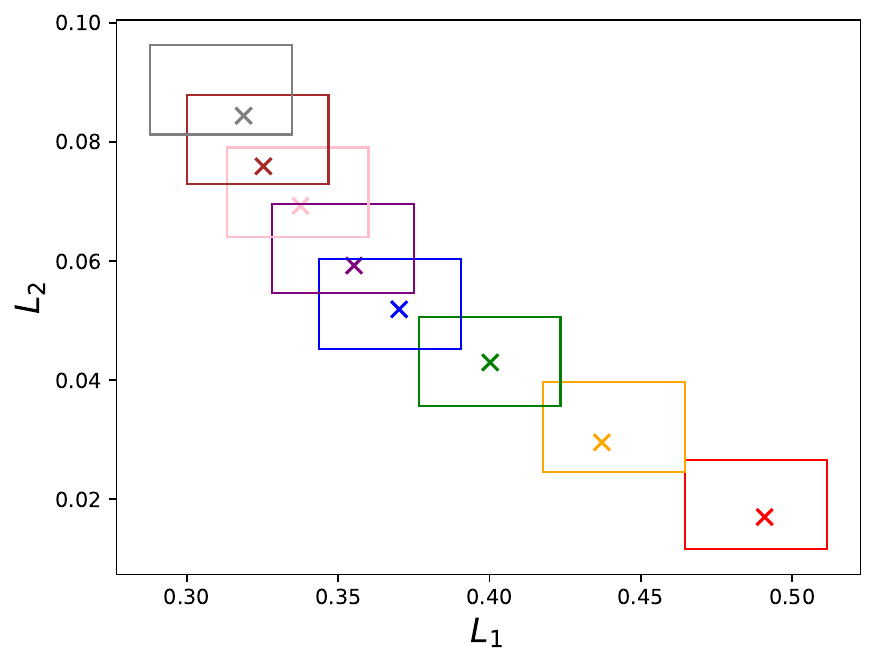}} 
    \sfig[Compas Box]{\ig[width= \hwdt \tw]{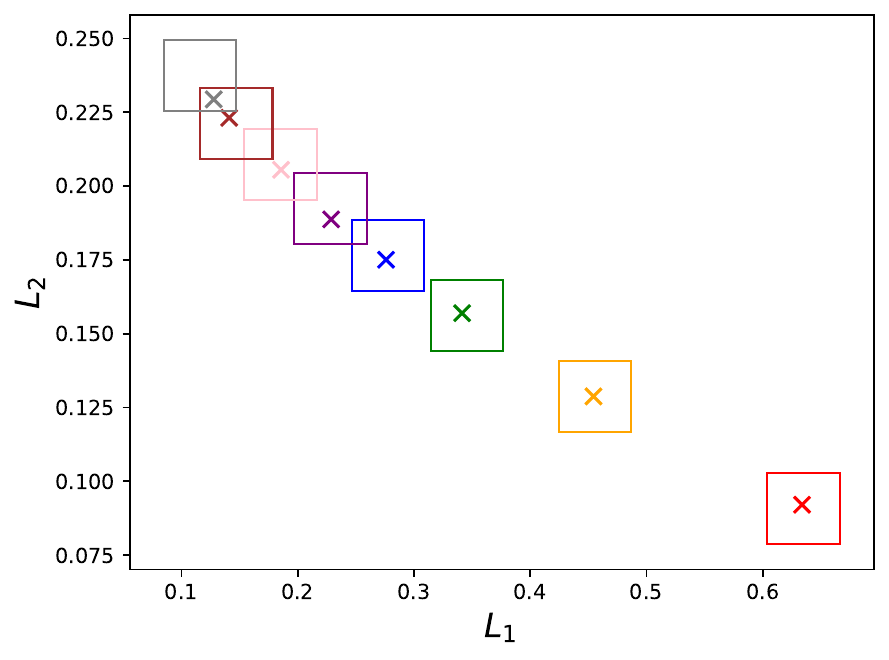}}
    \sfig[Credit Box]{\ig[width= \hwdt \tw]{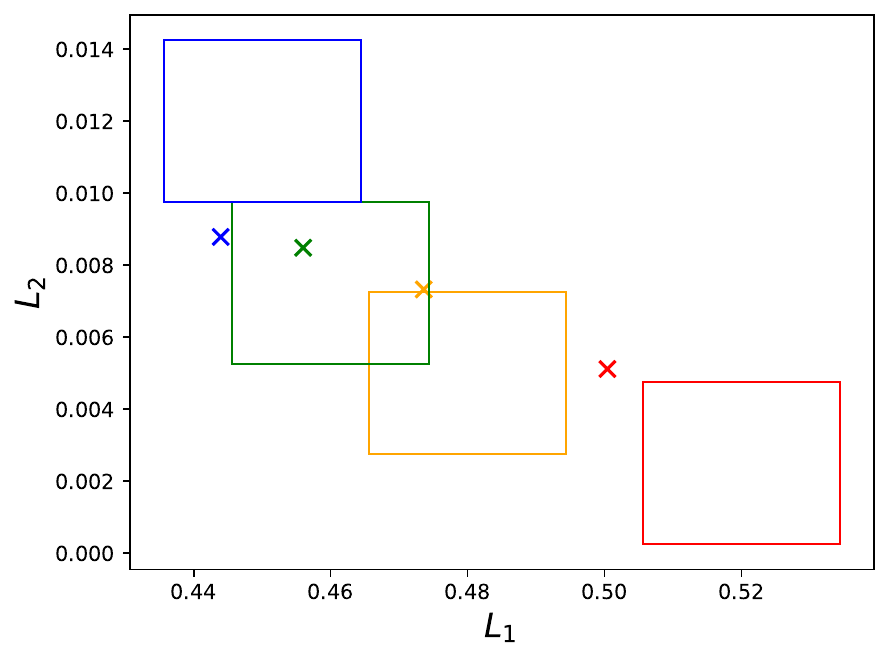}} 
    \caption{ROI constraints for fairness classification problems. PMGDA find Pareto solutions satisfying the ROI constraint on most preferences. PMGDA is the only method working for ROI constraint. } \label{fig:roi:fair}
\end{figure}

\sssec{Setting}
We test the performance of our method on three well-known fairness classification dataset, namely, Adult \cite{asuncion2007uci}, Compas \cite{angwin2016machine}, and Credit \cite{yeh2009comparisons}. The first objective $L_1(\vth)$ is the cross entropy loss. And the second objective $L_2(\vth)$ is the DEO (Difference of Equality of Opportunity) loss \cite{padh2021addressing}. This multiobjective multitask is firstly studied in \cite{ruchte2021cosmos}. During training, batch size is 128. The classification model is a three-layer fully-connected model activated by ReLU. Problem and model details are shown in \Cref{tab:mtl:problem}.

\begin{table}[]
\centering
\caption{Details of fairness classification problems.} \label{tab:mtl:problem}
\begin{tabular}{lllll}
\toprule       
       & Features & Params & Train Size & Test Size \\
\midrule
Adult  & 88      & 6891   & 34188      & 9769      \\
Compas & 20      & 2811   & 4319       & 1235      \\
Credit & 90      &  7011      & 21000      & 6000     \\
\bottomrule
\end{tabular}
\end{table}

\begin{table}[]
\caption{Results on fairness classification problems. The best result is marked in bold, while the second best result is marked with an underline.} \label{tab:mtl:res}
\centering
\begin{tabular}{llrrrr}
\toprule
Method & Indicator & COSMOS & mTche & EPO & PMGDA \\
\midrule
Adult & $\vartheta$ $\downarrow$ & 3.54 & \underline{2.07} & \textbf{2.02} & 2.19 \\
 & PBI $\downarrow$ & 0.46 & \textbf{0.41} & 0.60 & \underline{0.42} \\
Compas & $\vartheta$ & 4.77 & 3.91 & \textbf{2.94} & \underline{3.00} \\
 & PBI & 0.51 & 0.48 & \textbf{0.41} & \underline{0.43} \\
Credit & $\vartheta$ & 1.89 & \underline{1.68} & \underline{1.68} & \textbf{1.46} \\
 & PBI & 0.48 & 0.48 & 0.48 & \textbf{0.48} \\
\bottomrule
\end{tabular}
\end{table}

\sssec{Results}
The numerical results are presented in \Cref{tab:mtl:res}. For a comparative analysis with other methods, refer to \Cref{fig:mtl:all}. Additionally, a specific comparison with EPO, the sole published work claiming to search for exact Pareto solutions, is detailed in \Cref{fig:mtl:epo}.

Analysis of these figures reveals that EPO's results are still unstable, aligning with the observations in MOO-SVGD paper \cite{liu2021profiling}. Specifically, on the Adult problem, EPO's solutions are dominated by PMGDA and other methods, indicating EPO's lack of convergence in this task. In the simpler Compas problem, which features fewer training data and a smaller model size, mTche, EPO, and PMGDA exhibit similar performance. However, in the more challenging Credit problem, as illustrated in \Cref{fig:mtl:all}(c), only PMGDA successfully identifies exact Pareto solutions across all preferences, consequently achieving a significantly broader Pareto front.

Furthermore, the COSMOS method faces a significant challenge in determining its weighting factor, which is a complex process. Another drawback is its lack of theoretical guarantee for reaching exact Pareto solutions. Given that optimal weighting factors vary across different preferences, constantly tuning its hyperparameters for each task becomes overly burdensome.

Another advantage of the proposed PMGDA method is its adaptable design of the preference function \(h(\vth)\). This method is notably the first to successfully identify Pareto solutions that meet specific region of interest (ROI) constraints, as demonstrated in \Cref{fig:roi:fair}. The challenge in applying ROI to multi-task learning problems lies in the stochastic nature of estimated objectives $L_i(\vth)$, which vary with batches. Given this estimation noise, we recommend using a smaller, more conservative ROI scale, as applied here (1/1.5).

\subsection{Multi-objective Reinforcement Learning (MORL)}
In this section, we evaluate the performance of our method on several deep Multi-Objective Reinforcement Learning (MORL) tasks, as described in Xu et al. (2020) \cite{xu2020prediction}. MORL tasks are notably challenging due to the difficulty in training and the typically noisy gradients produced by RL algorithms. This evaluation demonstrates our method's scalability and effectiveness in complex practical scenarios. We focus on the multiobjective Swimmer, HalfCheetah, and Ant tasks. In these tasks, the primary objective is to maximize forward velocity, while the secondary objective is to minimize energy consumption.

Our first finding, illustrated in \Cref{fig:morl_linf}, reveals that our proposed method successfully identifies Pareto solutions within user-defined Regions of Interest (ROIs). Notably, PMGDA is the first gradient-based method capable of meeting such specific demands. In contrast, previous ROI-focused methods relying on evolutionary algorithms, such as those mentioned in \cite{reinaldo2021incorporation, filatovas2020reference, li2018integration}, have been unable to efficiently optimize policy networks in these tasks.  

\begin{figure}%
    \centering
    \sfig[Results on HalfCheetah. ]{\ig[width= \fwdt \tw]{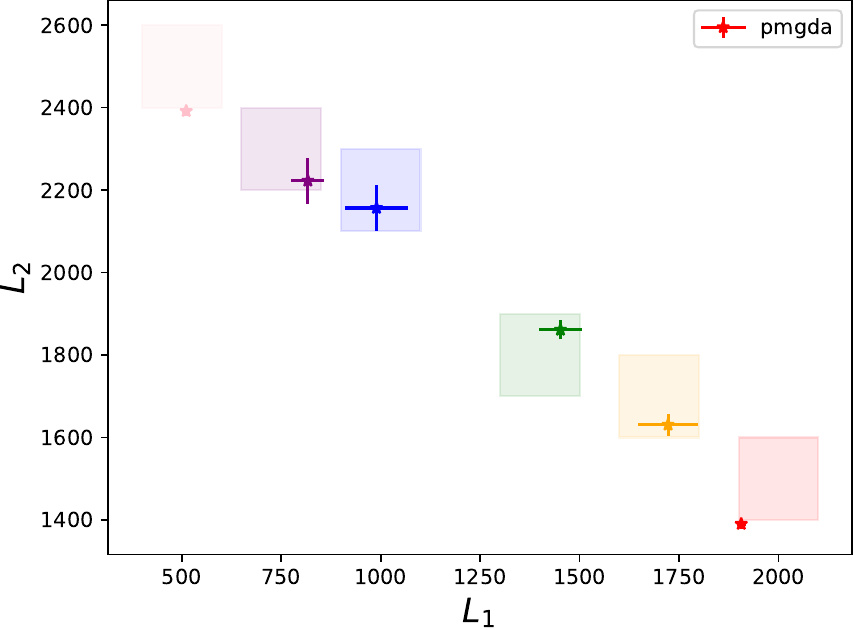}}
    \sfig[Results on Swimmer.]{\ig[width= \fwdt \tw]{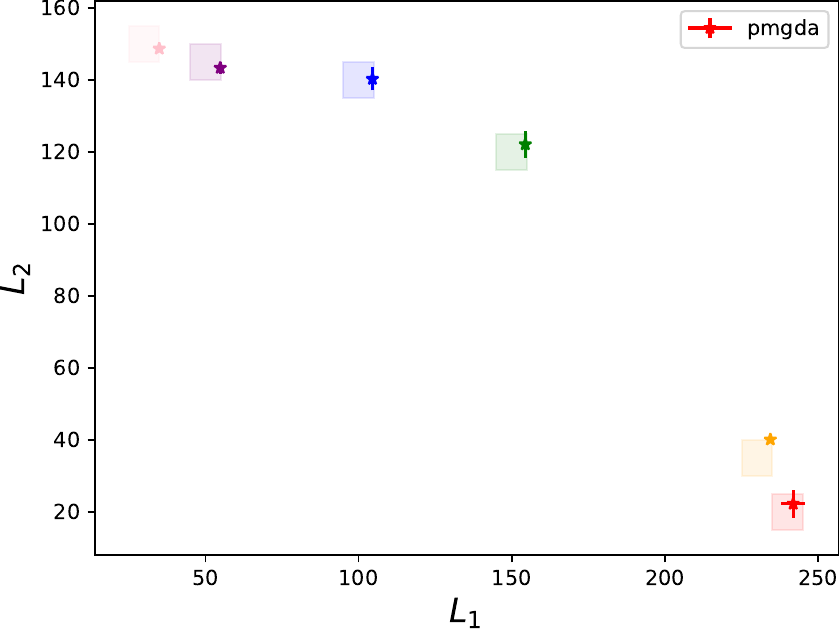}}
    \caption{Average RL results from three independent seeds, showcasing Pareto solutions adhering to ROI constraints in MORL tasks, presented with error bars. } \label{fig:morl_linf}
\end{figure}

We now present another noteworthy observation from our study on identifying specific preferences in the multiobjective Ant problem and give the insight behind this. In this problem, the first and second objectives are rapid leftward and forward movements. The preference vectors under examination, depicted as black dashed lines, are illustrated in \Cref{fig:ant}. Notably, the EPO method, being designed for minimization problems with positive objectives, is not ideally suited for this task (maximization task). Furthermore, tuning COSMOS for this scenario is even challenging due to the negligible impact of the cosine similarity term compared to the convergence term. Therefore, we have not included results from these two methods in our analysis. 

\begin{figure}
    \centering
    \includegraphics[width=0.4 \tw]{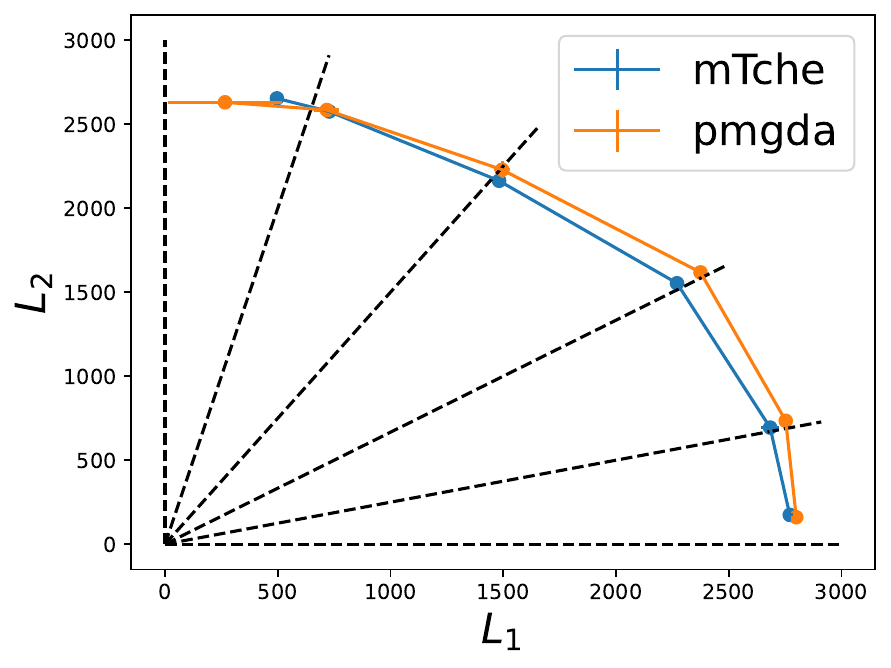}
    \caption{Exact solutions for the MO-Ant problem, averaged from the final 10 iterations' evaluations, displayed with error bars. PMGDA results outperform mTche results on most subproblems with different preferences.}
    \label{fig:ant}
\end{figure}

Our PMGDA method not only achieves faster convergence than the modified Tchebycheff (mTche) method in three-objective synthetic problems (refer to \Cref{section_synthetic}) but also surpasses mTche in terms of final performance. In practice, minimax optimization can be unstable and hard for
neural networks. As illustrated in \Cref{fig:ant}, the majority of PMGDA solutions are dominant over those of mTche. This superiority is primarily due to mTche’s bi-layer optimization structure, which is a max-min problem. In this setup, rapid switching of the arg min index among Pareto objectives occurs frequently after a solution has became `exact', leading to a zig-zag convergence trajectory. Consequently, the 'effective optimization quantities' for both objectives are reduced. Given that optimizing a single objective in RL is already challenging, this oscillation further detracts from mTche’s performance.

Conversely, the PMGDA method, particularly in its correction step, adopts a more flexible approach. It only necessitates that the update angle aligns with the negative gradient of the preference function while satisfying the Armijo condition. This approach makes the objectives are optimized as effectively as possible (see Equation \ref{eqn:correct}). This softer, more adaptable method of handling constraints has proven to yield significantly better final performances, especially in challenging tasks.

\section{Conclusions, Limitations, and Future Works} 
\label{sec_conclusion}
This paper addresses a fundamental problem in MOO: satisfying user-specific preference functions. Unlike previous approaches such as EPO \cite{pmlr-v119-mahapatra20a}, COSMOS \cite{ruchte2021cosmos}, and PMTL \cite{lin2019pareto}, which rely on specific rules to identify Pareto solutions under preference constraints, our proposed framework is versatile enough to incorporate general preference functions. In the well-known task of finding Pareto solutions that precisely align with a preference vector, our PMGDA method surpasses previous methods like EPO and COSMOS in precision, and outperforms mTche in terms of convergence speed. Additionally, our framework is capable of undertaking new challenges, such as identifying Pareto solutions that meet region of interest constraints—a task rarely addressed in gradient-based MOOs (MOOs).

However, the proposed method is not without limitations. As a gradient-based approach, it cannot guarantee globally optimal solutions. Future work will explore integrating gradient-based methods with evolutionary approaches to circumvent local optima and enhance solution quality.

\bibliographystyle{Style/IEEEtran}
\bibliography{pmgda,problem}

\clearpage
\end{document}